\documentclass[11pt]{article}
\usepackage{gd}
\usepackage{fullpage}

\def\stability{\ensuremath{\kappa}}
\def\hyp{\ensuremath{w}}
\def\hypother{\ensuremath{v}}
\def\ntrain{\ensuremath{n}}
\def\ntest{\ensuremath{m}}
\def\idx{\ensuremath{t}}
\def\otheridx{\ensuremath{s}}
\def\regret{\ensuremath{\mathfrak{R}}} %
\def\mgrv{\ensuremath{X}} %
\def\mg{\ensuremath{Z}} %
\def\indexset{\ensuremath{\mc{I}}}

\def\loss{\ensuremath{\ell}}
\def\liploss{\ensuremath{L}}
\def\lipobj{\ensuremath{G}}
\def\strongparam{\ensuremath{\lambda}}
\def\statsamplebound{\ensuremath{r}}
\def\outprodmat{\ensuremath{A}}
\def\underlyingmeasure{\ensuremath{\mu}}
\def\statlabelspace{\ensuremath{\mc{V}}}
\providecommand{\cov}{\mathop{\rm Cov}}
\renewcommand{\xdomain}{\ensuremath{\mathcal{W}}}

\renewcommand{\statsample}{\ensuremath{x}}
\renewcommand{\statsamplespace}{\ensuremath{\mc{X}}}
\renewcommand{\statlabel}{\ensuremath{y}}
\renewcommand{\statlabelspace}{\ensuremath{\mc{Y}}}

\usepackage[numbers]{natbib}

\makeatletter
\long\def\@makecaption#1#2{
        \vskip 0.8ex
        \setbox\@tempboxa\hbox{\small {\bf #1:} #2}
        \parindent 1.5em  %
        \dimen0=\hsize
        \advance\dimen0 by -3em
        \ifdim \wd\@tempboxa >\dimen0
                \hbox to \hsize{
                        \parindent 0em
                        \hfil 
                        \parbox{\dimen0}{\def\baselinestretch{0.96}\small
                                {\bf #1.} #2
                                } 
                        \hfil}
        \else \hbox to \hsize{\hfil \box\@tempboxa \hfil}
        \fi
        }
\makeatother

\title{The Generalization Ability of \\ Online Algorithms for Dependent
  Data}
\author{Alekh Agarwal ~~~~~~~~~~~~ John C.\ Duchi \\
  {\tt\{alekh,jduchi\}@eecs.berkeley.edu}}

\begin{document}

\maketitle

\begin{abstract}
  We study the generalization performance of online learning
  algorithms trained on samples coming from a dependent source of data.  We
  show that the generalization error of any stable online algorithm
  concentrates around its regret---an easily computable statistic of the
  online performance of the algorithm---when the underlying ergodic process is
  $\beta$- or $\phi$-mixing. We show high probability error bounds assuming
  the loss function is convex, and we also establish sharp convergence rates
  and deviation bounds for strongly convex losses and several linear
  prediction problems such as linear and logistic regression, least-squares
  SVM, and boosting on dependent data.  In addition, our results have
  straightforward applications to stochastic optimization with dependent data,
  and our analysis requires only martingale convergence arguments; we need not
  rely on more powerful statistical tools such as empirical process theory.
\end{abstract}

\section{Introduction}

Online learning algorithms have the attractive property that regret
guarantees---performance of the sequence of points $\hyp(1), \ldots,
\hyp(\ntrain)$ the online algorithm plays measured against a fixed
predictor $\hyp^*$---hold for arbitrary sequences of loss functions,
without assuming any statistical regularity of the sequence.  It is
natural to ask whether one can say something stronger when some
probabilistic structure underlies the sequence of examples, or loss
functions, presented
to the online algorithm. In particular, if the sequence of examples
are generated by a stochastic process, can the online learning
algorithm output a good predictor for future samples from the
same process?

When data is drawn independently and identically distributed from a
fixed underlying distribution, Cesa-Bianchi et
al.~\cite{CesaBianchiCoGe04} have shown that online learning
algorithms can in fact output predictors with good generalization
performance. Specifically, they show that for convex loss functions,
the average of the $\ntrain$ predictors played by the online algorithm
has---with high probability---small generalization error on future
examples generated i.i.d.\ from the same distribution.  In this paper,
we ask the same question when the data is drawn according to a
(dependent) ergodic process.

In addition, this paper helps provide justification for the use of
regret to a \emph{fixed} comparator $\hyp^*$ as a measure of
performance for online learning algorithms. Regret to a fixed
predictor is sometimes not a natural metric, which has led several
researchers to study online algorithms with performance guarantees for
(slowly) changing comparators $\hyp^*(1), \hyp^*(2), \ldots$ (see,
e.g., Herbster and Warmuth~\cite{HerbsterWa98,HerbsterWa01}). When
data comes i.i.d.\ from a (unknown) distribution, however,
online-to-batch conversions~\cite{CesaBianchiCoGe04} justify computing
regret with respect to a fixed $\hyp^*$. In this paper, we show that
even when data comes from a dependent stochastic process, regret to a
fixed comparator is both meaningful and a reasonable evaluation
metric.

Though practically, many settings require learning with
non-i.i.d.\ data---examples include time series data from financial
problems, meteorological observations, and learning for predictive
control---the generalization performance of statistical learning
algorithms for non-independent data is perhaps not so well understood
as that for the independent scenario.
In spite of natural difficulties encountered with dependent data, several
researchers have studied the convergence of statistical procedures in
non i.i.d.\ settings~\cite{Yu94,Meir00,ZouLiXu09,MohriRo10}. In such
scenarios, one generally assumes that the data are drawn from a
stationary $\alpha$-, $\beta$-, or $\phi$-mixing sequence, which
implies that dependence between observations weakens suitably over
time. Yu~\cite{Yu94} adapts classical empirical process techniques to
prove uniform laws of large numbers for dependent data; perhaps a more
direct parent to our approach is the work of Mohri and
Rostamizadeh~\cite{MohriRo10}, who combine algorithmic
stability~\cite{BousquetEl02} with known concentration inequalities to
derive generalization bounds.  Steinwart and
Christmann~\cite{SteinwartCh09} show fast rates of convergence for
learning from stationary geometrically $\alpha$-mixing processes, so
long as the loss functions satisfy natural localization and
self-bounding assumptions. Such assumptions were previously exploited
in the machine learning and statistics literature for independent
sequences (e.g.~\cite{BartlettBoMe05}), and Steinwart and Christmann
extend these results by building off Bernstein-type inequalities for
dependent sequences due to Modha and Masry~\cite{ModhaMa96}.

In this paper, we show that online learning algorithms enjoy
guarantees on generalization to unseen data for dependent data
sequences from $\beta$- and $\phi$-mixing sources. In particular, we
show that stable online learning algorithms---those that do not change
their predictor too aggressively between iterations---also yield
predictors with small generalization error. In the most favorable
regime of geometric mixing, we demonstrate generalization error on the
order of $\order(\log \ntrain / \sqrt{\ntrain})$ after training on
$\ntrain$ samples when the loss function is convex and Lipschitz. We
also demonstrate faster $\order(\log n/n)$ convergence when the loss
function is strongly convex in the hypothesis $\hyp$, which is the
usual case for regularized losses. In addition, we consider linear
prediction settings, and show $\order(\log n/n)$ convergence a loss
that is strongly convex in its scalar argument (though not in the
predictor $\hyp$) is applied to a linear predictor $\<\hyp,\cdot\>$,
which gives fast rates for least squares SVMs, least squares
regression, logistic regression, and boosting over bounded sets.  We
also provide an example and associated learning algorithm for which
the expected regret goes to $-\infty$, while any fixed predictor has
expected loss zero; this shows that low regret alone is not sufficient
to guarantee small expected error when data samples are dependent.

In demonstrating generalization guarantees for online learning algorithms
with dependent data,
we answer an open problem posed by
Cesa-Bianchi et al.~\cite{CesaBianchiCoGe04} on whether online
algorithms give good performance on unseen data when said data is drawn
from a mixing stationary process. Our results also answer a question
posed by Xiao~\cite{Xiao10} regarding the convergence of the
regularized dual averaging algorithm with dependent stochastic gradients.
More broadly, our results establish that any suitably stable optimization or
online learning algorithm converges in stochastic approximation settings when
the noise sequence is mixing. There is a rich history of classical work in
this area (see e.g.\ the book~\cite{KushnerYi03} and references therein), but
most results for dependent data are asymptotic, and to our knowledge there is
a paucity of finite sample and high probability convergence guarantees.  The
guarantees we provide have applications to, for example, learning from Markov
chains, autoregressive processes, or learning complex statistical models for
which inference is expensive~\cite{WeiTa90}.

Our techniques build off of a recent paper by Duchi et
al.~\cite{DuchiAgJoJo11}, where we
show high probability bounds on the convergence
of the mirror descent algorithm for stochastic optimization even when the
gradients are non-i.i.d. In particular, we build on our earlier martingale
techniques, showing concentration inequalities for dependent random variables
that are sharper than previously used Bernstein concentration for
geometrically $\alpha$-mixing processes~\cite{ModhaMa96,SteinwartCh09} by
exploiting recent ideas of Kakade and Tewari~\cite{KakadeTe09}, though we use
weakened versions of $\phi$-mixing and $\beta$-mixing to prove our high
probability results. Further, our proof techniques require only relatively
elementary martingale convergence arguments, and we do not require that the
input data is stationary but only that it is suitably convergent.

\section{Setup, Assumptions, and Notation}
\label{sec:setup}

We assume that the online algorithm receives $\ntrain$ data points
$\statsample_1, \dots, \statsample_{\ntrain}$ from a sample space
$\statsamplespace$, where the data is generated according to a
stochastic process $\statprob$, though the samples $\statsample_\idx$
are not necessarily i.i.d.\ or even independent. The online algorithm
plays points (hypotheses) $\hyp \in \xdomain$, and at iteration $\idx$
the algorithm plays the point $\hyp(\idx)$ and suffers the loss
$F(\hyp(\idx); \statsample_\idx)$.  We assume that the statistical
samples $\statsample_\idx$ have a stationary distribution
$\stationary$ to which they converge (we make this precise shortly),
and we measure generalization performance with respect to the expected
loss or risk functional
\begin{equation}
  \label{eqn:objective}
  f(\hyp) \defeq \E_\stationary[F(\hyp; \statsample)]
  = \int_{\statsamplespace} F(\hyp; \statsample) d\stationary(\statsample).
\end{equation}
Essentially, our goal is to show that after $\ntrain$ iterations of
any low-regret online algorithm, it is possible to use $\hyp(1),
\ldots, \hyp(\ntrain)$ to output a predictor or hypothesis
$\what{\hyp}_\ntrain$ for which $f(\what{\hyp}_\ntrain)$ is guaranteed
to be small with respect to any other hypothesis $\hyp^*$.

Discussion of our statistical assumptions requires a few additional
definitions.  The total variation distance between distributions $P$
and $Q$ defined on the probability space $(S, \mc{F})$ where $\mc{F}$
is a $\sigma$-field, each with densities $p$ and $q$ with respect to
an underlying measure $\underlyingmeasure$,\footnote{This assumption
  is without loss, since $P$ and $Q$ are each absolutely continuous
  with respect to the measure $P + Q$.} is given by
\begin{equation}
  \label{eqn:def-total-variation}
  \dtv(P, Q) \defeq \sup_{A \in \mc{F}} |P(A) - Q(A)|
  = \half \int_S |p(s) - q(s)| d\underlyingmeasure(s).
\end{equation}
Define the $\sigma$-field $\mc{F}_{\idx} = \sigma(\statsample_1,
\ldots, \statsample_{\idx})$. Let $\statprob^{\idx}_{[\otheridx]}$
denote the distribution of $\statsample_{\idx}$ conditioned on
$\mc{F}_{\otheridx}$, that is, given the initial samples
$\statsample_1, \ldots, \statsample_{\otheridx}$. Written slightly
differently, $\statprob^\idx_{[\otheridx]} = \statprob^\idx(\cdot \mid
\mc{F}_\otheridx)$ is a version of the conditional probability of
$\statsample_\idx$ given the sigma field $\mc{F}_\otheridx =
\sigma(\statsample_1, \ldots, \statsample_\otheridx)$. Our main
assumption is that the stochastic process is suitably mixing:
there is a stationary distribution $\stationary$ to which the
distribution of $\statsample_\idx$ converges as $\idx$ grows.  We also
assume that the distributions $\statprob^\idx_{[s]}$ and $\stationary$
are absolutely continuous with respect to an underlying measure
$\underlyingmeasure$ throughout. We use the following to measure
convergence:
\begin{definition}[Weak $\beta$ and $\phi$-mixing]
  \label{def:phi-mixing}
  The $\beta$ and $\phi$-mixing coefficients of the sampling
  distribution $\statprob$ are defined, respectively, as
  \begin{equation*}
    \beta(k) \defeq \sup_{t \in \N} \left\{ 2\E[\dtv(\statprob^{t +
        k}(\cdot \mid \mc{F}_t), \stationary)]\right\} ~~~ \mbox{and}
    ~~~ \phi(k) \defeq \sup_{t \in \N, B \in \mc{F}_t} \left\{ 2
    \dtv(\statprob^{t + k}( \cdot \mid B), \stationary)\right\}.
  \end{equation*}
\end{definition}
We say that the process is $\phi$-mixing (respectively,
$\beta$-mixing) if $\phi(k) \rightarrow 0$ ($\beta(k) \rightarrow 0$)
as $k \rightarrow \infty$, and we assume without loss that $\beta$ and
$\phi$ are non-increasing. The above definitions are weaker than the
standard definitions of mixing~\cite{MohriRo10,Bradley05,Yu94}, which
require mixing over the entire future $\sigma$-field of the process,
that is, $\sigma(\statsample_\idx, \statsample_{\idx + 1},
\statsample_{\idx + 2}, \ldots)$. In contrast, we require mixing over
only the single-slice marginal of $\statsample_{t + k}$. From the
definition, we also see that $\beta$-mixing is weaker than
$\phi$-mixing since $\beta(k) \le \phi(k)$. We state our results in general
forms using either the $\beta$ or $\phi$-mixing coefficients of the
stochastic process, and we generally use $\phi$-mixing results
for stronger high-probability guarantees compared to $\beta$-mixing.
We remark that if the sequence $\{\statsample_\idx\}$ is i.i.d., then
$\phi(1) = \beta(1) = 0$.

Two regimes of $\beta$-mixing (and $\phi$-mixing) will be of special
interest. A process is called geometrically $\beta$-mixing
($\phi$-mixing) if $\beta(k) \leq \beta_0\exp(-\beta_1 k^{\mixexp})$
(respectively $\phi(k) \leq \phi_0\exp(-\phi_1k^{\mixexp})$) for some
$\beta_i, \phi_i, \mixexp > 0$. Some stochastic processes satisfying
geometric mixing include finite-state ergodic Markov chains and a large
class of aperiodic, Harris-recurrent Markov processes; see
the references~\cite{MeynTw09,ModhaMa96} for more examples.
A process is called algebraically $\beta$-mixing ($\phi$-mixing) if
$\beta(k) \leq \beta_0 k^{-\mixexp}$ (resp.\ $\phi(k) \leq
\phi_0 k^{-\mixexp}$) for constants $\beta_0, \phi_0, \mixexp > 0$. Examples of
algebraic mixing arise in certain Metropolis-Hastings samplers when
the proposal distribution does not have a lower bounded
density~\cite{JarnerRo02}, some queuing systems, and other
unbounded processes.

We now turn to stating the relevant assumptions on the instantaneous loss
functions $F(\cdot; \statsample)$ and other quantities relevant to the online
learning algorithm.  Recall that the algorithm plays points
(hypothesis) $\hyp \in \xdomain$.  Throughout, we make the following
boundedness assumptions on $F$ and the domain $\xdomain$, which are common in
the online learning literature.
\begin{assumption}[Boundedness]
  \label{assumption:F-lipschitz}
  For $\underlyingmeasure$-almost every (henceforth $\underlyingmeasure$-a.e.)
  $\statsample$, the function $F(\cdot; \statsample)$ is convex and
  $\lipobj$-Lipschitz with respect to a norm $\norm{\cdot}$ over
  $\xdomain$:
  \begin{equation}
    \label{eqn:F-lipschitz}
    |F(\hyp; \statsample) - F(\hypother; \statsample)| \le \lipobj
    \norm{\hyp - \hypother}
  \end{equation}
  for all $\hyp,\hypother \in \xdomain$. In addition, $\xdomain$ is
  compact and has finite radius: for any $\hyp, \hyp^* \in \xdomain$,
  \begin{equation}
    \label{eqn:compactness}
    \norm{\hyp - \hyp^*} \le \radius.
  \end{equation}
  Further, $F(\hyp; \statsample) \in [0, \lipobj \radius]$.
\end{assumption}
\noindent
As a consequence of Assumption~\ref{assumption:F-lipschitz} $f$ is
also $\lipobj$-Lipschitz.  Given the first two
bounds~\eqref{eqn:F-lipschitz} and~\eqref{eqn:compactness} of
Assumption~\ref{assumption:F-lipschitz}, the final condition can
be assumed without loss; we make it explicit to avoid centering issues later.
In the sequel, we give somewhat stronger results in the
presence of the following additional assumption, which lower bounds
the curvature of the expected function $f$:
\begin{assumption}[Strong convexity]
  \label{assumption:strongly-convex}
  The expected function $f$ is $\strongparam$-strongly convex with respect
  to the norm $\norm{\cdot}$, that is,
  \begin{equation}
    \label{eqn:strongly-convex}
    f(\hypother) \ge f(\hyp) + \<g, \hypother - \hyp\> +
    \frac{\strongparam}{2} \norm{\hyp - \hypother}^2 ~~~ \mbox{for} ~
    \hyp,\hypother \in \xdomain~ \mbox{and for all}~g \in \partial
    f(\hyp).
  \end{equation}
\end{assumption}

Lastly, to prove generalization error bounds for online learning
algorithms, we require them to be appropriately stable, as described
in the next assumption.
\begin{assumption}
  \label{assumption:stability}
  There is a non-increasing sequence $\stability(\idx)$ such that if
  $\hyp(\idx)$ and $\hyp(\idx+1)$ are successive iterates of the online
  algorithm, then $\norm{\hyp(\idx) - \hyp(\idx+1)} \le
  \stability(\idx)$.
\end{assumption}
\noindent
Here $\norm{\cdot}$ is the same norm as that used in
Assumption~\ref{assumption:F-lipschitz}. We observe that this
stability assumption is different from the stability condition of
Mohri and Rostamizadeh~\cite{MohriRo10} and neither one implies the
other. It is common (or at least straightforward) to establish bounds
$\stability(t)$ as a part of the regret analysis of online algorithms
(e.g.~\cite{Xiao10}), which motivates our assumption here.

What remains to complete our setup is to quantify our assumptions on
the performance of the online learning algorithm. We assume access to
an online algorithm whose regret is bounded by (the possibly random
quantity) $\regret_{\ntrain}$ for the sequence of points
$\statsample_1, \ldots, \statsample_\ntrain \in
\statsamplespace$, that is, the online algorithm produces a sequence
of iterates $\hyp(1), \ldots, \hyp(\ntrain)$ such that for any fixed
$\hyp^* \in \xdomain$,
\begin{equation}
  \sum_{\idx=1}^{\ntrain} F(\hyp(\idx); \statsample_\idx)
  - F(\hyp^*, \statsample_\idx) \leq
  \regret_{\ntrain}.
  \label{eqn:lowregret}
\end{equation}

Our goal is to use the sequence $\hyp(1), \ldots, \hyp(\ntrain)$ to construct
an estimator $\what{\hyp}_\ntrain$ that performs well on unseen data. Since
our samples are dependent, we measure the generalization error on future test
samples drawn from the same sample path as the training
data~\cite{MohriRo10}. That is, we measure performance on the $\ntest$ samples
$\statsample_{\ntrain+1}, \ldots, \statsample_{\ntrain+\ntest}$ drawn from the
process $\statprob_{[\ntrain]}$, and we would like to bound the future risk of
$\what{\hyp}_\ntrain$, defined as
\begin{equation}
  \frac{1}{\ntest}\sum_{\idx = 1}^{\ntest}
  \E\left[F(\what{\hyp}_{\ntrain}; \statsample_{\ntrain+\idx}) -
    F(\hyp^*; \statsample_{\ntrain+\idx}) \mid
    \mc{F}_\ntrain \right],
  \label{eqn:test-error}
\end{equation}
the conditional expectation of the losses $F(\what{\hyp}_\ntrain;
\statsample)$ given the first $\ntrain$ samples. Note that in the
i.i.d.\ setting~\cite{CesaBianchiCoGe04}, the expectation above is the excess
risk $f(\what{\hyp}_\ntrain) - f(\hyp^*)$ of $\what{\hyp}_\ntrain$ against
$\hyp^*$, because $\statsample_{\ntrain + \idx}$ is independent of
$\statsample_1, \ldots, \statsample_\ntrain$.  Of course, we are in the
dependent setting, so the generalization measure~\eqref{eqn:test-error}
requires slightly more care.

1\section{Generalization bounds for convex functions} 

Our definitions and assumptions in place, we show in this section that
any suitably stable online learning algorithm enjoys a
high-probability generalization guarantee for convex loss functions
$F$.  The main results of this section are
Theorems~\ref{theorem:highprob-error-convex}
and~\ref{theorem:highprob-error-convex-beta}, which give high
probability convergence of any \emph{stable online learning algorithm}
under $\phi$- and $\beta$-mixing, respectively.  Following
Theorem~\ref{theorem:highprob-error-convex}, we also present an
example illustrating that low regret is by itself insufficient to
guarantee good generalization performance, which is distinct from
i.i.d.\ settings~\cite{CesaBianchiCoGe04}.

Before proceeding with our technical development, we describe the
high-level structure and intuition underlying our proofs. The
technical insight underpinning many of our results is that under our
mixing assumptions, the distribution of the random instance
$\statsample_{\idx + \tau}$ is close to the stationary distribution
conditioned on $\mc{F}_\idx$. That is, looking some number of steps
$\tau$ into the futre from a time $\idx$ is almost as good as
obtaining an unbiased sample from the stationary distribution
$\stationary$. As a result, the loss
$F(\hyp(\idx);\statsample_{\idx+\tau})$ is a good proxy for
$f(\hyp(t))$, since $\hyp(t)$ only depends on $\statsample_1, \ldots,
\statsample_{\idx-1}$. Lemma~\ref{lemma:lookahead-mixing} formalizes
this intuition. (Duchi et al.~\cite{DuchiAgJoJo11} use a similar
technique as a building block.) Under our stability condition, we can
further demonstrate that $F(\hyp(\idx);\statsample_{\idx+\tau})$ is
close to $F(\hyp(\idx+\tau);\statsample_{\idx+\tau})$, and the
behavior of the latter sequence is nearly the same as the sequence
$F(\hyp(\idx); \statsample_\idx)$ with respect to which the regret
$\regret_\ntrain$ is measured.  We make these these ideas formal in
Propositions~\ref{prop:stationary-to-test}
and~\ref{proposition:master-theorem}.  We then combine our
intermediate results (including bounds on the regret
$\regret_\ntrain$), applying relevant martingale concentration
inequalities, to obtain the main theorems of this and later sections.

Our starting point is the above-mentioned technical lemma that
underlies many of our results.
\begin{lemma}
  \label{lemma:lookahead-mixing}
  Let $\hyp, \hypother \in \xdomain$ be measurable with
  respect to the $\sigma$-field $\mc{F}_\idx$ and
  Assumption~\ref{assumption:F-lipschitz} hold. Then for any $\tau
  \in \N$,
  \begin{equation*}
    \E[F(\hyp; \statsample_{\idx + \tau}) -
      F(\hypother; \statsample_{\idx + \tau})
      \mid \mc{F}_\idx]
    \le f(\hyp) - f(\hypother) + \lipobj \radius \phi(\tau).
  \end{equation*}
  and
  \begin{equation*}
    \E\Big[ \big|\E[F(\hyp; \statsample_{\idx + \tau})
        - F(\hypother; \statsample_{\idx + \tau})
        \mid \mc{F}_\idx]
      - (f(\hyp) - f(\hypother)) \big|\Big]
    \le \lipobj \radius \beta(\tau).
  \end{equation*}
\end{lemma}
\begin{proof}
  We first prove the result for the $\phi$-mixing bound.
  Recalling that $f(\hyp) =
  \E_\stationary[F(\hyp; \statsample)]$ and the definition of the
  underlying measure $\underlyingmeasure$ and the densities
  $\stationarydensity$ and $\statdensity$,
  \begin{align*}
    \E[F(\hyp; \statsample_{\idx + \tau}) -
      F(\hypother; \statsample_{\idx + \tau}) \mid \mc{F}_\idx]
    & = 
    \E[F(\hyp; \statsample_{\idx + \tau}) - f(\hyp)
      + f(\hypother) - 
      F(\hypother; \statsample_{\idx + \tau}) \mid \mc{F}_\idx]
    + f(\hyp) - f(\hypother) \\
    & = \int_\statsamplespace [F(\hyp; \statsample) -
    F(\hypother; \statsample)] (\statdensity_{[\idx]}^{\idx + \tau}(
    \statsample)
    - \stationarydensity(\statsample)) d \underlyingmeasure(\statsample)
    + f(\hyp) - f(\hypother) \\
    & \le \int_\statsamplespace \left|F(\hyp; \statsample) -
    F(\hypother; \statsample)\right|
    \left|\statdensity_{[\idx]}^{\idx + \tau}(\statsample)
    - \stationarydensity(\statsample)\right|
    d \underlyingmeasure(\statsample) + f(\hyp) - f(\hypother) \\
    & \le \lipobj \radius \int
    \left|\statdensity_{[\idx]}^{\idx + \tau}(\statsample)
    - \stationarydensity(\statsample)\right|
    d \underlyingmeasure(\statsample) + f(\hyp) - f(\hypother) \\
    & = 2 \lipobj \radius \cdot
    \dtv(\statprob_{[\idx]}^{\idx + \tau}, \stationary)
    + f(\hyp) - f(\hypother),
  \end{align*}
  where for the second inequality we used the Lipschitz
  assumption~\ref{assumption:F-lipschitz} and the compactness
  assumption on $\xdomain$. Noting that $2
  \dtv(\statprob_{[\idx]}^{\idx + \tau}, \stationary) \le \phi(\tau)$
  by the definition~\ref{def:phi-mixing} completes the proof of the
  first part.
  
  To see the second inequality using $\beta$-mixing coefficients,
  we begin by noting that as a consequence of the proof of the first
  inequality,
  \begin{equation*}
    \E[F(\hyp; \statsample_{\idx + \tau}) -
      F(\hypother; \statsample_{\idx + \tau}) \mid \mc{F}_\idx]
    - (f(\hyp) - f(\hypother))
    \le 2 \lipobj \radius \dtv(\statprob_{[\idx]}^{\idx + \tau}, \stationary),
  \end{equation*}
  and the inequality holds with $\hyp$ and $\hypother$ switched:
  \begin{equation*}
    \E[F(\hypother; \statsample_{\idx + \tau}) -
      F(\hyp; \statsample_{\idx + \tau}) \mid \mc{F}_\idx]
    - (f(\hypother) - f(\hyp))
    \le 2 \lipobj \radius \dtv(\statprob_{[\idx]}^{\idx + \tau}, \stationary).
  \end{equation*}
  Combining the two inequalities and taking expectations, we have
  \begin{equation*}
    \E\Big[ \big|\E[F(\hyp; \statsample_{\idx + \tau})
        - F(\hypother; \statsample_{\idx + \tau})
        \mid \mc{F}_\idx]
      - (f(\hyp) - f(\hypother)) \big|\Big]
    \le 2 \lipobj \radius \E\left[\dtv(\statprob^{\idx + \tau}(\cdot
      \mid \mc{F}_\idx), \stationary)\right]
    \le \lipobj \radius \beta(\tau)
  \end{equation*}  
  by the definition~\ref{def:phi-mixing} of the mixing coefficients.
\end{proof}

Using Lemma~\ref{lemma:lookahead-mixing}, we can give a proposition
that relates the risk on the test sequence to the expected error of a
predictor $\hyp$ under the stationary distribution. The result shows
that for any $\hyp$ measurable with respect to the $\sigma$-field
$\mc{F}_{\ntrain}$---we use $\what{\hyp}_\ntrain \in
\mc{F}_{\ntrain}$, the (unspecified as yet) output of the online
learning algorithm---we can prove generalization bounds by showing
that $\hyp$ has small risk under the stationary distribution
$\stationary$.
\begin{proposition}
  \label{prop:stationary-to-test}
  Under the Lipschitz assumption~\ref{assumption:F-lipschitz}, for any
  $\hyp \in \xdomain$ measurable with respect to $\mc{F}_\ntrain$, any
  $\hyp^* \in \mc{W}$, and any $\tau \in \N$,
  \begin{equation*}
    \frac{1}{\ntest}\sum_{\idx = \ntrain + 1}^{\ntrain + \ntest}
    \E\left[ F(\hyp;\statsample_{\idx}) - F(\hyp^*;\statsample_{\idx})
      \mid \mc{F}_{\ntrain}\right] \le f(\hyp) - f(\hyp^*) +
    \phi(\tau) \lipobj\radius + \frac{(\tau-1) \lipobj\radius}{\ntest}
  \end{equation*}
  and
  \begin{equation*}
    \E\bigg[\frac{1}{\ntest} \sum_{\idx = \ntrain + 1}^{\ntrain +
        \ntest} \E\left[F(\hyp;\statsample_{\idx}) -
        F(\hyp^*;\statsample_{\idx}) \mid \mc{F}_{\ntrain}\right]
      \bigg] \le \E[f(\hyp)] - f(\hyp^*) + \beta(\tau) \lipobj\radius
    + \frac{(\tau-1) \lipobj\radius}{\ntest}.
  \end{equation*}
\end{proposition}
\begin{proof}
  The proof follows from the definition~\ref{def:phi-mixing} of mixing.
  The key idea is to give up on the first
  $\tau-1$ test samples and use the mixing assumption to control the
  loss on the remainder. We have
  \begin{align*}
    \lefteqn{\sum_{\idx = \ntrain+1}^{\ntrain + \ntest}
      \E[F(\hyp;\statsample_{\idx}) - F(\hyp^*;\statsample_{\idx})
        \mid \mc{F}_{\ntrain}]} \\
    & = \sum_{\idx = \ntrain+1}^{\ntrain + \tau-1}
    \E[F(\hyp;\statsample_{\idx}) - F(\hyp^*;\statsample_{\idx})
      \mid \mc{F}_{\ntrain}] + \sum_{\idx =
      \ntrain+ \tau}^{\ntrain + \ntest} \E[F(\hyp;\statsample_{\idx})
      - F(\hyp^*;\statsample_{\idx}) \mid \mc{F}_{\ntrain-1}] \\
    & \le (\tau-1) \lipobj \radius
    + \sum_{\idx = \ntrain + \tau}^{\ntrain + \ntest}
    \E[F(\hyp;\statsample_{\idx})
      - F(\hyp^*;\statsample_{\idx}) \mid \mc{F}_{\ntrain}]
  \end{align*}
  since by the Lipschitz assumption~\ref{assumption:F-lipschitz} and
  compactness $F(\hyp; \statsample) - F(\hyp^*; \statsample) \le \lipobj
  \radius$.
  Now, we
  apply Lemma~\ref{lemma:lookahead-mixing} to the summation, which
  completes the proof.
\end{proof}

Proposition~\ref{prop:stationary-to-test} allows us to focus on
controlling the error on the expected function $f$ under the
stationary distribution $\stationary$, which is a natural convergence
guarantee. Indeed, the function $f$ is the risk functional with
respect to which convergence is measured in the standard i.i.d.\ case,
and applying Proposition~\ref{prop:stationary-to-test} with $\tau = 1$
and $\phi(1) = 0$ (or $\beta(1) = 0$) confirms that the bound is equal
to $f(\hyp) - f(\hyp^*)$.  We now turn to controlling the error under
$f$, beginning with a result that relates risk performance of the
sequence of hypotheses $\hyp(1), \ldots, \hyp(\ntrain)$ output by the
online learning algorithm to the algorithm's regret, a term dependent
on the stability of the algorithm, and an additional random term. This
proposition is the starting point for the remainder of our results in
this section.
\begin{proposition}
  \label{proposition:master-theorem}
  Let Assumptions~\ref{assumption:F-lipschitz} and~\ref{assumption:stability}
  hold and let $\hyp(\idx)$ denote the
  sequence of outputs of the online algorithm. Then
  for any $\tau \in \N$,
  \begin{align}
    \sum_{\idx = 1}^\ntrain f(\hyp(\idx)) - f(\hyp^*)
    & \le \regret_\ntrain + \lipobj \tau \sum_{\idx = 1}^\ntrain
    \stability(\idx) + 2\tau \lipobj \radius \nonumber \\
    & \qquad ~ + \sum_{\idx = 1}^{\ntrain}
    \left[f(\hyp(\idx)) - F(\hyp(\idx); \statsample_{\idx + \tau})
    + F(\hyp^*; \statsample_{\idx + \tau}) - f(\hyp^*)\right].
    \label{eqn:key-inequality}
  \end{align}
\end{proposition}
\begin{proof}
  We begin by expanding the regret of $\hyp(\idx)$ on sequence $f$ via
  \begin{align}
    \lefteqn{\sum_{\idx = 1}^\ntrain [f(\hyp(\idx)) - f(\hyp^*)]}
    \nonumber \\
    & = \sum_{\idx = 1}^\ntrain \left[f(\hyp(\idx)) -
    F(\hyp(\idx); \statsample_{\idx + \tau})
    + F(\hyp(\idx); \statsample_{\idx + \tau}) - f(\hyp^*)\right] \nonumber \\
    & = \sum_{\idx = 1}^\ntrain \left[f(\hyp(\idx)) -
    F(\hyp(\idx); \statsample_{\idx + \tau})
    + F(\hyp^*; \statsample_{\idx + \tau}) - f(\hyp^*)
    + F(\hyp(\idx); \statsample_{\idx + \tau}) -
      F(\hyp^*; \statsample_{\idx + \tau})\right].
    \label{eqn:simple-regret-no-prob}
  \end{align}
  Now we use stability and the regret guarantee~\eqref{eqn:lowregret}
  to bound the last two terms of the
  summation~\eqref{eqn:simple-regret-no-prob}.  To that end, note that
  \begin{align*}
    \lefteqn{\sum_{\idx = 1}^\ntrain \left[F(\hyp(\idx);
        \statsample_{\idx + \tau}) - F(\hyp^*; \statsample_{\idx +
          \tau})\right]}\\ 
    &= \underbrace{\sum_{\idx = 1}^\ntrain
        \left[F(\hyp(\idx); \statsample_\idx) - F(\hyp^*;
          \statsample_\idx)\right]}_{\term_1} +
    \underbrace{\sum_{\idx = 1}^{\ntrain - \tau} \left[F(\hyp(\idx);
        \statsample_{\idx + \tau}) - F(\hyp(\idx + \tau);
        \statsample_{\idx + \tau})\right]}_{\term_2}\\ 
    & + \underbrace{ \sum_{\idx = \ntrain - \tau +
        1}^\ntrain F(\hyp(\idx); \statsample_{\idx + \tau}) -
      \sum_{\idx = 1}^\tau F(\hyp(\idx); \statsample_\idx) +
      \sum_{\idx=1}^\tau F(\hyp^*;\statsample_\idx) -
      \sum_{\idx=\ntrain+1}^{\ntrain + \tau}
      F(\hyp^*;\statsample_\idx)}_{\term_3}. 
  \end{align*}
  We now bound the three terms in the summation. $\term_3$ is bounded
  by $2\tau \lipobj \radius$ under the boundedness
  assumption~\ref{assumption:F-lipschitz}, and the regret
  bound~\eqref{eqn:lowregret} guarantees that $\term_1 \le
  \regret_\ntrain$. Using the stability
  assumption~\ref{assumption:stability}, we can bound $\term_2$ by
  noting
  \begin{equation*}
    F(\hyp(\idx); \statsample_{\idx + \tau})
    - F(\hyp(\idx + \tau); \statsample_{\idx + \tau})
    \le \lipobj \norm{\hyp(\idx) - \hyp(\idx + \tau)}
    \le \lipobj \sum_{s = 0}^{\tau - 1} \stability(\idx + s)
    \le \lipobj \tau \stability(\idx),
  \end{equation*}
  where the last step uses the non-increasing property of the
  coefficients $\stability(\idx)$. 
  Substituting the bounds on $\term_1$, $\term_2$, and $\term_3$ into
  Eq.~\eqref{eqn:simple-regret-no-prob} completes the
  proposition.
\end{proof}

The remaining development of this section consists of using the key
inequality~\eqref{eqn:key-inequality} in
Proposition~\ref{proposition:master-theorem} to give expected and
high-probability convergence guarantees for the online learning
algorithm. Throughout, we define the output of the online algorithm to
be the averaged predictor
\begin{equation}
  \what{\hyp}_\ntrain = \frac{1}{\ntrain}
  \sum_{\idx = 1}^\ntrain \hyp(\idx).
  \label{eqn:avg-predictor}
\end{equation}
We begin with results giving convergence in expectation for stable
online algorithms.
\begin{theorem}
  \label{theorem:regret-to-expected}
  Under Assumptions~\ref{assumption:F-lipschitz}
  and~\ref{assumption:stability}, for any $\tau \in \N$ the predictor
  $\what{\hyp}_{\ntrain}$ satisfies the guarantee
  \begin{equation*}
    \E[f(\what{\hyp}_{\ntrain})] - f(\hyp^*) \leq \ninv
    \E[\regret_\ntrain] + \beta(\tau) \lipobj\radius + \frac{(\tau-1)
      \lipobj}{\ntrain} \bigg(2\radius +
    \sum_{\idx=1}^{\ntrain}\stability(\idx)\bigg),
  \end{equation*}
  for any $\hyp^* \in \mc{W}$.
\end{theorem} 
\begin{proof}
  From the inequality~\eqref{eqn:key-inequality} in
  Proposition~\ref{proposition:master-theorem}, what remains is to
  take the expectation of the random quantities.  To that end, we note
  that $\hyp(\idx)$ is measurable with respect to $\mc{F}_{\idx-1}$
  (since the iterate at time $\idx$ depends only on first $\idx-1$
  samples) and apply Lemma~\ref{lemma:lookahead-mixing}, which gives
  \begin{equation*}
    \E\left[\E[F(\hyp^*; \statsample_{\idx + \tau-1})
        - F(\hyp(\idx); \statsample_{\idx + \tau-1})
        \mid \mc{F}_{\idx-1}]\right]
    \le f(\hyp^*) - f(\hyp(\idx))
    + \lipobj \radius \beta(\tau).
  \end{equation*}
  Adding the difference to the sum~\eqref{eqn:key-inequality} with the
  setting $\tau\mapsto (\tau - 1)$ gives
  \begin{equation*}
    \E\bigg[\sum_{\idx = 1}^\ntrain f(\hyp(\idx)) - f(\hyp^*)\bigg]
    \le \E[\regret_\ntrain]
    + \lipobj (\tau-1) \sum_{\idx = 1}^\ntrain \stability(\idx)
    + 2(\tau-1) \lipobj \radius
    + \ntrain \lipobj \radius \beta(\tau).
  \end{equation*}
  Dividing by $\ntrain$ and observing that
  $f(\what{\hyp}_{\ntrain}) \leq
  \frac{1}{\ntrain}\sum_{\idx=1}^{\ntrain} f(\hyp(\idx))$ by Jensen's
  inequality completes the proof. 
\end{proof}

We observe that setting $\tau = 1$ and $\beta(1) = 0$ recovers an
expected version of the results of Cesa-Bianchi et al.~\cite[Corollary
  2]{CesaBianchiCoGe04} for
i.i.d.\ samples. Theorem~\ref{theorem:regret-to-expected} combined
with Proposition~\ref{prop:stationary-to-test} immediately yields the
following generalization bound. Our other results can be similarly
extended, but we leave such development to the reader.
\begin{corollary}
  \label{cor:exp-error-convex}
  Under Assumptions~\ref{assumption:F-lipschitz}
  and~\ref{assumption:stability}, for any $\tau \in \N$ the predictor
  $\what{\hyp}_{\ntrain}$ satisfies the guarantee
  \begin{equation*}
    \frac{1}{\ntest} \E\bigg[\sum_{\idx=\ntrain+1}^{\ntrain+\ntest}
      F(\what{\hyp}_{\ntrain};\statsample_{\idx}) -
      F(\hyp^*;\statsample_{\idx})\bigg]
    \le \ninv \E[\regret_\ntrain] + 2 \beta(\tau) \lipobj\radius +
    (\tau-1) \lipobj\radius
    \bigg(\frac{2}{\ntrain} + \frac{1}{\ntest}
    + \frac{1}{\ntrain}\sum_{\idx=1}^{\ntrain}\stability(\idx)\bigg).
  \end{equation*}
\end{corollary} 

It is clear that the stability assumption we make on the online
algorithm plays a key role in our results whenever $\tau > 1$, that
is, the samples are indeed dependent. It is natural to ask whether this
additional term is just an artifact of our analysis, or whether
low-regret by itself ensures a small error under the stationary
distribution even for dependent data.
The next example shows that low regret---by itself---is insufficient
for generalization guarantees, so some additional assumption on the
online algorithm is necessary to guarantee small error under the
stationary distribution.

\begin{example}[Low-regret does not imply convergence]
  \label{example:bad}
  In $1$-dimension, define the linear loss $F(\hyp;\statsample) =
  \ip{\hyp}{\statsample}$, where $\statsample \in \{-1,1\}$ and the
  set $\xdomain = [-1, 1]$. Let $p > 0$ and define following
  dependent sampling process: at each time $\idx$, set
  \begin{equation*}
    \statsample_\idx = \left\{\begin{array}{cl}
    1 & \mbox{with probability}~ p/2 \\
    -1 & \mbox{with probability} ~ p/2 \\
    \statsample_{\idx-1} & \mbox{with probability} ~ 1-p.
    \end{array} \right.
  \end{equation*}
  The stationary distribution $\stationary$ is
  uniform on $\{-1,1\}$, so the expected error $\E_\stationary
  [\ip{\hyp}{\statsample}] = 0$ for any $\hyp \in \xdomain$.
  However, we can demonstrate an update rule with negative expected regret
  as follows. Consider the algorithm which sets $\hyp(\idx) =
  -\statsample_{\idx-1}$, implementing a trivial so-called \emph{follow the
    leader} strategy. With probability $1-p/2$, the
  value $\ip{\hyp(\idx)}{\statsample_\idx} = -1$,
  while $\ip{\hyp(\idx)}{\statsample_\idx} = 1$ with probability $p/2$.
  Consequently, the expectation of the cumulative sum
  $\sum_{\idx=1}^\ntrain F(\hyp(\idx); \statsample_\idx)$ is
  $-(1-p)\ntrain$.
  Using standard results on the expected deviation of
  the simple random walk (e.g.~\cite{Billingsley86}), we know that 
  \begin{equation*}
    \E\left[\inf_{\hyp \in \xdomain} \sum_{\idx=1}^\ntrain
    \ip{\hyp}{\statsample_\idx}\right]
    = -\E\left|\sum_{\idx=1}^\ntrain
    \statsample_\idx\right| = \Theta(-\sqrt{\ntrain}).
  \end{equation*}
  We are thus
  guaranteed that the expected regret of the update rule is
  $-\Omega((1-p)\ntrain)$.
\end{example}

We have now seen that it is possible to achieve guarantees on the
generalization properties of an online learning algorithm by taking
expectation over both the training and test samples. We would like to
prove stronger results that hold with high probability over the
training data, as is possible in
i.i.d.\ settings~\cite{CesaBianchiCoGe04}. The next theorem applies
martingale concentration arguments using the Hoeffding-Azuma
inequality~\cite{Azuma67} to give high-probability concentration for
the random quantities remaining in
Proposition~\ref{proposition:master-theorem}'s bound.
\begin{theorem}
  \label{theorem:highprob-error-convex}
  Under Assumptions~\ref{assumption:F-lipschitz}
  and~\ref{assumption:stability}, with probability at least
  $1-\delta$, for any $\tau \in \N$ and any $\hyp^* \in \mc{W}$ the
  predictor $\what{\hyp}_{\ntrain}$ satisfies the guarantee
  \begin{equation*}
    f(\what{\hyp}_{\ntrain}) - f(\hyp^*) \leq
    \ninv \regret_{\ntrain} + \frac{(\tau-1) \lipobj}{\ntrain}
    \sum_{\idx=1}^\ntrain \stability(\idx) +
    2\lipobj\radius\sqrt{\frac{2\tau}{\ntrain} \log\frac{\tau}{\delta}} +
    \phi(\tau) \lipobj\radius
    + \frac{2(\tau-1) \lipobj \radius}{\ntrain}.
  \end{equation*}
\end{theorem}
\begin{proof}
  Inspecting the inequality~\eqref{eqn:key-inequality} from
  Proposition~\ref{proposition:master-theorem}, we observe that
  it suffices to bound
  \begin{equation}
    \mg_{\ntrain} \defeq \sum_{\idx = 1}^{\ntrain}
    \left[f(\hyp(\idx)) -
    f(\hyp^*) - F(\hyp(\idx);\statsample_{\idx+\tau-1}) +
    F(\hyp^*;\statsample_{\idx+\tau-1}) \right]
    \label{eqn:mgdefn-comb}
  \end{equation}
  This is analogous to the term that arises in the
  i.i.d.\ case~\cite{CesaBianchiCoGe04}, where $\mg_\ntrain$ is a
  bounded martingale sequence and hence concentrates around its
  expectation.  Our proof that the sum~\eqref{eqn:mgdefn-comb}
  concentrates is similar to the argument Duchi et
  al.~\cite{DuchiAgJoJo11} use to prove concentration for the ergodic
  mirror descent algorithm.  The idea is that though $\mg_{\ntrain}$
  is not quite a martingale in the general ergodic case, it is in fact
  a sum of $\tau$ \emph{near}-martingales. This technique of using
  blocks of random variables in dependent settings has also been used
  in previous work to directly bound the moment generating function of
  sums of dependent variables~\cite{ModhaMa96}, though our approach is
  different. See Fig.~\ref{fig:mixing-blocks} for a graphical
  representation of our choice~\eqref{eqn:def-mg-rv} of the martingale
  sequences.

  \begin{figure}[t]
    \begin{center}
      \includegraphics[width=.6\columnwidth]{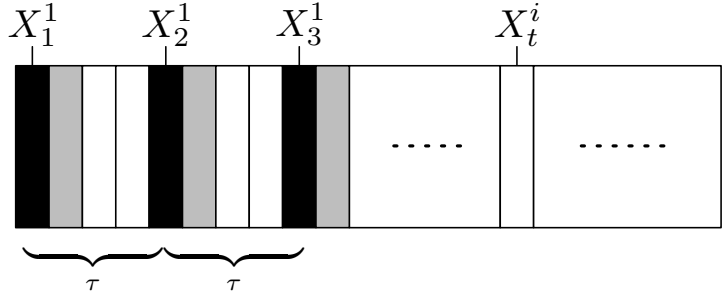}
    \end{center}
    \caption{\label{fig:mixing-blocks} The $\tau$ different blocks of
      near-martingales used in the proof of
      Theorem~\ref{theorem:highprob-error-convex}. Black boxes
      represent elements in the same index set $\indexset(1)$, gray in
      $\indexset(2)$, and so on.}
  \end{figure}

  For $i \in \{1, \ldots, \tau\}$ and $\idx \in \{1, \ldots,
  \ceil{\ntrain/\tau}\}$, define the random variables
  \begin{equation}
    \mgrv_\idx^i \defeq f(\hyp((\idx - 1) \tau + i)) - f(\hyp^*) +
    F(\hyp^*; \statsample_{\idx \tau + i-1}) - F(\hyp((\idx - 1) \tau
    + i); \statsample_{\idx \tau + i - 1}).
    \label{eqn:def-mg-rv}
  \end{equation}
  In addition, define the associated $\sigma$-fields $\mc{F}_\idx^i
  \defeq \mc{F}_{\idx \tau + i - 1} = \sigma(\statsample_1, \ldots,
  \statsample_{\idx\tau + i-1})$.  Then it is clear that $\mgrv_\idx^i$
  is measurable with respect to $\mc{F}_\idx^i$ (recall that
  $\hyp(\idx)$ is measurable with respect to $\mc{F}_{\idx-1}$), so
  the sequence
  $\mgrv_\idx^i - \E[\mgrv_\idx^i \mid \mc{F}_{\idx - 1}^i]$
  defines a martingale difference sequence adapted to the filtration
  $\mc{F}_\idx^i$, $\idx = 1, 2, \ldots$.  Following previous subsampling
  techniques~\cite{ModhaMa96,DuchiAgJoJo11}, we define the index set
  $\indexset(i)$ to be the indices $\{1, \ldots, \floor{\ntrain / \tau} + 1\}$
  for $i \le \ntrain - \tau \floor{\ntrain/\tau}$ and $\{1, \ldots,
  \floor{\ntrain/\tau}\}$ otherwise. Then a bit of algebra
  shows that
  \begin{equation}
    \mg_\ntrain = \sum_{i=1}^\tau \sum_{\idx \in \indexset(i)}
    \left[\mgrv_\idx^i - \E[\mgrv_\idx^i \mid \mc{F}^i_{\idx - 1}]\right]
    + \sum_{i=1}^\tau \sum_{\idx \in \indexset(i)}
    \E[\mgrv_\idx^i \mid \mc{F}^i_{\idx - 1}].
    \label{eqn:martingale-representation}
  \end{equation}

  The first term in the
  decomposition~\eqref{eqn:martingale-representation} is a sum of
  $\tau$ different martingale difference sequences.  In addition, the
  boundedness assumption~\ref{assumption:F-lipschitz} guarantees that
  $|\mgrv_\idx^i - \E[\mgrv_\idx^i \mid \mc{F}_{\idx-1}^i]| \le 2
  \lipobj \radius$, so each of the sequences is a bounded difference
  sequence. The Hoeffding-Azuma inequality~\cite{Azuma67} then
  guarantees
  \begin{equation}
    \P\left[\sum_{\idx \in \indexset(i)} \left[\mgrv_\idx^i -
        \E[\mgrv_\idx^i \mid \mc{F}^i_{\idx - 1}] \right] \ge \gamma
      \right] \le \exp\left(-\frac{\tau \gamma^2}{8\ntrain
      \lipobj^2\radius^2} \right).
    \label{eqn:azuma}
  \end{equation}
  To control the expectation term from the second sum in the
  representation~\eqref{eqn:martingale-representation}, we use mixing.
  Indeed, Lemma~\ref{lemma:lookahead-mixing} immediately implies that
  $\E[\mgrv_\idx^i \mid \mc{F}_{\idx - 1}^i] \le \lipobj \radius
  \phi(\tau)$.
  Combining these bounds with the application~\eqref{eqn:azuma} of
  Hoeffding-Azuma inequality, we see by a union bound that
  \begin{equation*}
    \P\left(\mg_\ntrain > \ntrain \lipobj \radius \phi(\tau)
    + \gamma\right) \le \sum_{i=1}^\tau
    \P\bigg[\sum_{\idx \in \indexset(i)}\left[ \mgrv_\idx^i -
        \E[\mgrv_\idx^i \mid \mc{F}^i_{\idx - 1}]\right] \ge \gamma/\tau \bigg]
    \le \tau\exp\left(-\frac{\gamma^2}{8 \tau n \lipobj^2
      \radius^2}\right).
  \end{equation*}
  Equivalently, by setting $\gamma = 2 \lipobj \radius \sqrt{2 \ntrain
    \tau \log(\tau / \delta)}$, we obtain that with probability at
  least $1 - \delta$,
  \begin{equation*}
    \mg_\ntrain \le \lipobj \radius \left(\ntrain\phi(\tau)
    + 2 \sqrt{2 \ntrain \tau \log\frac{\tau}{\delta}}\right).
  \end{equation*}
  Dividing by $\ntrain$ and using the convexity of $f$ as in
  the proof of Theorem~\ref{theorem:regret-to-expected} completes the proof.
\end{proof}

To better illustrate our results, we now specialize them under
concrete mixing assumptions in several corollaries, which should make
clearer the rates of convergence of the procedures.  We begin with two
corollaries giving generalization error bounds for geometrically and
algebraically $\phi$-mixing processes (defined in
Section~\ref{sec:setup}).

\begin{corollary}
  \label{cor:convex-geometric}
  Under the assumptions of
  Theorem~\ref{theorem:highprob-error-convex}, assume further that
  $\phi(k) \le \const\exp(-\phi_1 k^{\mixexp})$ for some universal
  constant $\const$. There exists a finite universal constant $C$ such
  that with probability at least $1-\delta$, for any $\hyp^* \in
  \xdomain$
  \begin{equation*}
    f(\what{\hyp}_{\ntrain}) - f(\hyp^*) \le
    \frac{1}{\ntrain} \regret_\ntrain +
    C \cdot \bigg[
      \frac{(\log \ntrain)^{1/\mixexp}\lipobj}{\ntrain\phi_1^{1/\mixexp}}
      \sum_{\idx=1}^{\ntrain}
  \stability(\idx) + \lipobj\radius \sqrt{\frac{(\log
      \ntrain)^{1/\mixexp}}{\ntrain \phi_1^{1/\mixexp}}\log\frac{(\log
      \ntrain)^{1/\mixexp}}{\delta}}\bigg].
  \end{equation*}
\end{corollary}
\noindent
The corollary follows from Theorem~\ref{theorem:highprob-error-convex} by
taking $\tau = (\log \ntrain/(2\phi_1))^{1/\mixexp}$.
When the samples $\statsample_\idx$ come from a geometrically
$\phi$-mixing process, Corollary~\ref{cor:convex-geometric} yields
a high-probability generalization bound of the same order as that in the
i.i.d.\ setting~\cite{CesaBianchiCoGe04} up to poly-logarithmic
factors.  Algebraic mixing gives somewhat slower rates:
\begin{corollary}
  \label{cor:convex-algebraic}
  Under the assumptions of
  Theorem~\ref{theorem:highprob-error-convex}, assume further that
  $\phi(k) \leq \phi_0k^{-\mixexp}$. Define $K_\ntrain = \sum_{\idx =
    1}^\ntrain \stability(\idx)/\radius$. There exists a finite
  universal constant $C$ such that with probability at least
  $1-\delta$, for any $\hyp^* \in \mc{W}$
  \begin{equation*}
    f(\what{\hyp}_{\ntrain}) - f(\hyp^*) \le \frac{1}{\ntrain}
    \regret_\ntrain + C \cdot \left[ \lipobj \radius\phi_0^{\frac{1}{1
          + \theta}}
      \left(\frac{K_\ntrain}{\ntrain}\right)^{\frac{\mixexp}{1 +
          \mixexp}}
      + \lipobj \radius \phi_0^{\frac{1}{\mixexp + 1}}
      \left(K_\ntrain \ntrain^\mixexp\right)^{\frac{-1}{2 \mixexp + 2}}
      \sqrt{\frac{1}{\mixexp + 1} \cdot \log \frac{n}{K_\ntrain \delta}}
      \right].
  \end{equation*}
\end{corollary}
\noindent
The corollary follows by setting $\tau = \phi_0^{1 / (\mixexp + 1)}
(\ntrain / K_\ntrain)^{1/(\mixexp+1)}$.  So long as the sum of the
stability constants $\sum_{\idx = 1}^\ntrain \stability(\idx) =
o(\ntrain)$, the bound in Corollary~\ref{cor:convex-algebraic}
converges to 0. In addition, we remark that under the same condition
on the stability, an argument similar to that for Corollary~7 of Duchi
et al.~\cite{DuchiAgJoJo11} implies $f(\what{\hyp}_\ntrain) -
f(\hyp^*) \rightarrow 0$ almost surely whenever $\phi(k) \rightarrow
0$ as $k\rightarrow \infty$.

To obtain concrete generalization error rates from our results, one
must know bounds on the stability sequence $\stability(\idx)$ (and the
regret $\regret_{\ntrain}$). For many online algorithms, the stability
sequence satisfies $\stability(\idx) \propto 1 / \sqrt{\idx}$,
including online gradient and mirror descent~\cite{DuchiShSiTe10}.  As
a more concrete example, consider Nesterov's dual averaging
algorithm~\cite{Nesterov09}, which Xiao extends to regularized
settings~\cite{Xiao10}. For convex, $\lipobj$-Lipschitz functions, the
dual averaging algorithm satisfies $\regret_\ntrain =
\order(\lipobj\radius\sqrt{\ntrain})$, and with appropriate stepsize
choice~\cite[Lemma 10]{Xiao10} proportional to $\sqrt{\idx}$, one has
$\stability(\idx) \le \radius / \sqrt{\idx}$.
Noting that $\sum_{\idx = 1}^\ntrain \idx^{-1/2} \le 2
\sqrt{\ntrain}$, substituting the stability bound into the
result of Theorem~\ref{theorem:highprob-error-convex} immediately
yields the following: there exists a universal constant $C$ such that
with probability at least $1-\delta$,
\begin{align}
 f(\what{\hyp}_{\ntrain}) - f(\hyp^*) & \le \frac{1}{\ntrain}
 \regret_\ntrain + C \cdot \inf_{\tau \in \N} \left[\frac{\lipobj
     \radius (\tau-1)}{\sqrt{\ntrain}} + \frac{\lipobj
     \radius}{\sqrt{\ntrain}} \sqrt{\tau \log \frac{\tau}{\delta}} +
   \phi(\tau) \lipobj \radius \right].
 \label{eqn:stable-online-corollary}
\end{align}
The bound~\eqref{eqn:stable-online-corollary} captures the known
convergence rates for
i.i.d.\ sequences~\cite{CesaBianchiCoGe04,Xiao10} by taking $\tau =
1$, since $\phi(1) = 0$ in i.i.d.\ settings.  In addition,
specializing to the geometric mixing rate of
Corollary~\ref{cor:convex-geometric} one obtains a generalization
error bound of $\order\left(\left(1 +
\frac{1}{\phi_1}\right)\frac{1}{\sqrt{\ntrain}}\right)$ to
poly-logarithmic factors.

Theorem~\ref{theorem:highprob-error-convex} and the corollaries following
require $\phi$-mixing of the stochastic sequence $\statsample_1,
\statsample_2, \ldots$, which is perhaps an undesirably strong assumption in
some situations (for example, when the sample space $\statsamplespace$ is
unbounded). To mitigate this, we now give high-probability convergence
results under the weaker assumption that the stochastic process $\statprob$ is
$\beta$-mixing. These results are (unsurprisingly) weaker than those for
$\phi$-mixing; nonetheless, there is no significant loss in rates of
convergence as long as the process $\statprob$ mixes quickly enough.
\begin{theorem}
  \label{theorem:highprob-error-convex-beta}
  Under Assumptions~\ref{assumption:F-lipschitz}
  and~\ref{assumption:stability}, with probability at least
  $1-2\delta$, for any $\tau \in \N$ and for all $\hyp^* \in \mc{W}$
  the predictor $\what{\hyp}_{\ntrain}$ satisfies the guarantee
  \begin{equation*}
    f(\what{\hyp}_{\ntrain}) - f(\hyp^*) \leq \ninv \regret_{\ntrain}
    + \frac{(\tau - 1)\lipobj}{\ntrain} \sum_{\idx=1}^\ntrain
    \stability(\idx) + 2\lipobj\radius\sqrt{\frac{2\tau}{\ntrain}
      \log\frac{2\tau}{\delta}} + \frac{2\beta(\tau)
      \lipobj\radius}{\delta} + \frac{2(\tau-1) \lipobj
      \radius}{\ntrain}.
  \end{equation*}
\end{theorem}
\begin{proof}
  Following the proof of Theorem~\ref{theorem:highprob-error-convex}, we
  construct the random variables $\mg_{\ntrain}$ and $\mgrv_{\idx}^i$ as in
  the definitions~\eqref{eqn:mgdefn-comb}
  and~\eqref{eqn:def-mg-rv}. Decomposing $\mg_{\ntrain}$ into the two part
  sum~\eqref{eqn:martingale-representation}, we similarly apply the
  Hoeffding-Azuma inequality (as in the proof of
  Theorem~\ref{theorem:highprob-error-convex}) to the first term. The
  treatment of the second piece requires more care.

  Observe that for any fixed $i,\idx$, the fact that $\hyp((\idx - 1)
  \tau + i)$ and $\hyp^*$ are measurable with respect to $\mc{F}_{\idx
    - 1}^i$ guarantees via Lemma~\ref{lemma:lookahead-mixing} that
  \begin{equation*}
    \E\left[\left|\E\left[\mgrv_\idx^i\mid\mc{F}_{\idx-1}^i\right]\right|
      \right] \le \lipobj\radius\beta(\tau).
  \end{equation*}
  Applying Markov's inequality, we see that with probability at least
  $1-\delta$,
  \begin{equation*}
    \sum_{i=1}^\tau\sum_{\idx \in
      \indexset(i)}\E\left[\mgrv_{\idx}^i\mid\mc{F}_{\idx-1}^i\right] \leq
    \frac{\ntrain\lipobj\radius \beta(\tau)}{\delta}.
  \end{equation*}
  Continuing as in the proof of Theorem~\ref{theorem:highprob-error-convex}
  yields the result of the theorem.
\end{proof}

Though the $1/\delta$ factor in
Theorem~\ref{theorem:highprob-error-convex-beta} may be large, we now
show that things are not so difficult as they seem. Indeed, let us now
make the additional assumption that the stochastic process
$\statsample_1, \statsample_2, \ldots$ is geometrically
$\beta$-mixing.  We have the following corollary.
\begin{corollary}
  \label{cor:convex-geometric-beta}
  Under the assumptions of
  Theorem~\ref{theorem:highprob-error-convex-beta}, assume further
  that $\beta(k) \le \beta_0\exp(-\beta_1 k^{\mixexp})$. There exists
  finite universal constant $C$ such that with probability at least
  $1-1/n$ for any $\hyp^* \in \mc{W}$
  \begin{equation*}
    f(\what{\hyp}_{\ntrain}) - f(\hyp^*) \le \frac{1}{\ntrain}
    \regret_\ntrain + C \cdot \left[ \frac{(1.5\log
        \ntrain)^{1/\mixexp}\lipobj}{\ntrain\beta_1^{1/\mixexp}}
      \sum_{\idx=1}^{\ntrain} \stability(\idx) + \lipobj\radius
      \sqrt{\frac{(1.5\log \ntrain)^{1/\mixexp}}{\ntrain
          \beta_1^{1/\mixexp}} \log \left(n (\log
        \ntrain)^{1/\mixexp}\right)} +
      \frac{\beta_0\lipobj\radius}{\sqrt{\ntrain}}\right].
  \end{equation*}
\end{corollary}
\noindent
The corollary follows from
Theorem~\ref{theorem:highprob-error-convex-beta} by setting $\tau =
(1.5\log \ntrain / \beta_1)^{1/\mixexp}$ and a few algebraic
manipulations.  Corollary~\ref{cor:convex-geometric-beta} shows that
under geometric $\beta$-mixing, we have essentially identical
high-probability generalization guarantees as we had for $\phi$-mixing
(cf.\ Corollary~\ref{cor:convex-geometric}), unless the desired error
probability or the mixing constant $\mixexp$ is extremely small. We
can make similar arguments for polynomially $\beta$-mixing stochastic
processes, though the associated weakening of the bound is somewhat
more pronounced.

\section{Generalization error bounds for strongly convex functions} 
\label{sec:strongly-convex}

It is by now well-known that the regret of online learning algorithms
scales as $\order(\log \ntrain)$ for strongly convex functions,
results which are due to work of Hazan et al.~\cite{HazanAgKa07}.  To
remind the reader, we recall
Assumption~\ref{assumption:strongly-convex}, which states that a
function $f$ is $\strongparam$-strongly convex with respect to the
norm $\norm{\cdot}$ if for all $g \in \partial f(\hyp)$,
\begin{equation*}
  f(\hypother) \geq f(\hyp) + \<g, \hypother - \hyp\> +
  \frac{\strongparam}{2} \norm{\hyp -
    \hypother}^2
  ~~~\mbox{for}~\hyp,\hypother \in \xdomain.
\end{equation*}

For many online algorithms, including online gradient and mirror
descent~\cite{BeckTe03,HazanAgKa07,ShalevSi07_tech,DuchiShSiTe10} and
dual averaging~\cite[Lemma 11]{Xiao10}, the iterates
satisfy the stability bound $\norm{\hyp(\idx) - \hyp(\idx + 1)} \le
\lipobj / (\strongparam \idx)$ when the loss functions
$F(\cdot,\statsample)$ are $\strongparam$-strongly convex. Under these
conditions, Corollary~\ref{cor:exp-error-convex} gives expected
generalization error bound of $\order(\inf_{\tau \in \N}
\left\{\beta(\tau) + \tau\log\ntrain/\ntrain\right\})$ as compared to
$\order(\inf_{\tau \in \N} \{\beta(\tau) + \sqrt{\tau / \ntrain}\})$
for non-strongly convex problems.  The improvement in rates, however,
does not apply to Theorem~\ref{theorem:highprob-error-convex}'s high
probability results, since the term controlling the fluctuations
around the expectation of the martingale we construct scales as
$\otil(\sqrt{\tau/n})$.  That said, when the samples
$\statsample_\idx$ are drawn i.i.d.\ from the distribution
$\stationary$, Kakade and Tewari~\cite{KakadeTe09} show a
generalization error bound of $\order(\log \ntrain/\ntrain)$ with high
probability by using self-bounding properties of an appropriately
constructed martingale. In the next theorem, we combine the techniques
used to prove our previous results with a self-bounding martingale
argument to derive sharper generalization guarantees when the expected
function $f$ is strongly convex. Throughout this section, we will
focus on error to the minimum of the expected function: $\hyp^* \in
\arg\min_{\hyp \in \xdomain} f(\hyp)$.
\begin{theorem}
  \label{theorem:highprob-error-strong}
  Let
  Assumptions~\ref{assumption:F-lipschitz},~\ref{assumption:strongly-convex},
  and~\ref{assumption:stability} hold, so the expected function $f$ is
  $\strongparam$-strongly convex with respect to the norm
  $\norm{\cdot}$ over $\xdomain$.  Then for any $\delta < 1/e$,
  $\ntrain \geq 3$, with probability at least $1-4\delta \log n$, for
  any $\tau \in \N$ the predictor $\what{\hyp}_{\ntrain}$ satisfies
  \begin{equation*}
    f(\what{\hyp}_{\ntrain}) - f(\hyp^*) \le \frac{2}{\ntrain}
    \regret_\ntrain + \frac{2 (\tau-1) \lipobj}{\ntrain}
    \left(\sum_{\idx = 1}^\ntrain \stability(\idx) + 2\radius\right) +
    \frac{32 \lipobj^2 \tau}{\strongparam \ntrain}
    \log\frac{\tau}{\delta} + \frac{12 \tau \radius
      \lipobj}{\ntrain}\log \frac{\tau}{\delta} + 2 \radius \lipobj
    \phi(\tau).
  \end{equation*}
\end{theorem}

Before we prove the theorem, we illustrate its use with a simple
corollary. We again use Xiao's extension of Nesterov's dual averaging
algorithm~\cite{Nesterov09,Xiao10}, where for $\lipobj$-Lipschitz
$\strongparam$-strongly convex losses $F$ it is shown that
\begin{equation*}
  \norm{x(\idx) - x(\idx + 1)} \le \stability(\idx) \le \frac{\lipobj}{
    \strongparam \idx}.
\end{equation*}
Consequently,
Theorem~\ref{theorem:highprob-error-strong} yields the following corollary,
applicable to dual averaging, mirror descent, and online gradient descent:
\begin{corollary}
  \label{corollary:strongly-convex}
  In addition to the conditions of
  Theorem~\ref{theorem:highprob-error-strong}, assume the stability bound
  $\stability(\idx) \le \lipobj / \strongparam \idx$. There is a universal
  constant $C$ such that with probability at least $1 - \delta \log \ntrain$,
  \begin{equation*}
    f(\what{\hyp}_\ntrain) - f(\hyp^*) \le \frac{2}{\ntrain}
    \regret_\ntrain + C \cdot \inf_{\tau \in \N} \left[ \frac{(\tau-1)
        \lipobj^2}{\strongparam \ntrain} \log \ntrain + \frac{\tau
        \lipobj^2}{\strongparam \ntrain} \log \frac{\tau}{\delta} +
      \frac{\lipobj^2}{\strongparam} \phi(\tau)\right].
  \end{equation*}
\end{corollary}
\begin{proof}
  The proof follows by noting the following two facts: first,
  $\sum_{\idx = 1}^\ntrain \stability(\idx) \le
  (\lipobj / \strongparam)(1 + \log \ntrain)$,
  and secondly, the definition~\eqref{eqn:strongly-convex} of strong
  convexity implies
  \begin{equation*}
    \lipobj \norm{\hyp - \hypother} \ge
    f(\hypother) - f(\hyp)
    \ge \<\nabla f(\hyp), \hypother - \hyp\>
    + \frac{\strongparam}{2}\norm{\hypother - \hyp}^2.
  \end{equation*}
  Recalling~\cite{HiriartUrrutyLe96ab} that $\dnorm{\nabla f(\hyp)} \le
  \lipobj$, we have
  $\norm{\hyp - \hypother} \le 4 \lipobj / \strongparam$ for
  all $\hyp, \hypother \in \xdomain$, so $\radius \le 2\lipobj/\strongparam$.
\end{proof}
We can further extend Corollary~\ref{corollary:strongly-convex} using
mixing rate assumptions on $\phi$ as in
Corollaries~\ref{cor:convex-geometric} and~\ref{cor:convex-algebraic},
though this follows the same lines as those.  For a few more concrete
examples, we note that online gradient and mirror descent as well as
dual averaging~\cite{HazanAgKa07,DuchiShSiTe10,ShalevSi07_tech,Xiao10}
all have $\regret_\ntrain \le C \cdot (\lipobj^2 / \strongparam) \log
\ntrain$ when the loss functions $F(\cdot; \statsample)$ are strongly
convex (this is stronger than assuming that the expected function $f$
is strongly convex, but it allows sharp logarithmic bounds on the
random quantity $\regret_\ntrain$). In this special case,
Corollary~\ref{corollary:strongly-convex} implies the generalization
bound
\begin{equation*}
  f(\what{\hyp}_\ntrain) - f(\hyp^*) =
  \order\left( \frac{\lipobj^2}{\strongparam}
  \inf_{\tau \in \N} \left[
    \tau \frac{\log \ntrain}{\ntrain} + \phi(\tau)\right]\right)
\end{equation*}
with high probability. For example, online algorithms for
SVMs (e.g.~\cite{ShalevSiSrCo11}) and other regularized problems
satisfy a sharp high-probability generalization guarantee, even for
non-i.i.d.\ data.

We now turn to proving Theorem~\ref{theorem:highprob-error-strong},
beginning with a martingale concentration inequality.
\begin{lemma}[Freedman~\cite{Freedman75},
    Kakade and Tewari~\cite{KakadeTe09}]
  \label{lemma:freedman}
  Let $X_1,\dots,X_{\ntrain}$ be a martingale difference sequence adapted
  to the filtration $\mc{F}_\idx$ with $|X_{\idx}| \leq b$. Define $V =
  \sum_{\idx=1}^{\ntrain} \E[\,X_{\idx}^2\,|\,\mc{F}_{\idx-1}\,]$. For
  any $\delta < 1/e$ and $\ntrain \geq 3$
  \begin{equation*}
    \P\left[ \sum_{\idx=1}^{\ntrain} X_{\idx} \geq
      \max\{2\sqrt{V},3b\sqrt{\log 1/\delta}\}\sqrt{\log
        1/\delta}\right] \leq 4\delta\log \ntrain. 
  \end{equation*}
\end{lemma}

\begin{proof-of-theorem}[\ref{theorem:highprob-error-strong}]
  For the proof of this theorem, we do not start from the
  Proposition~\ref{proposition:master-theorem}, as we did for the
  previous theorems, but begin directly with an appropriate
  martingale. Recalling the definition~\eqref{eqn:def-mg-rv} of the
  random variables $\mgrv_\idx^i$ and the $\sigma$-fields
  $\mc{F}_\idx^i = \sigma(\statsample_1, \ldots, \statsample_{\idx
    \tau + i - 1})$ from the proof of
  Theorem~\ref{theorem:highprob-error-convex}, our goal will be to
  give sharper concentration results for the martingale difference
  sequence $\mgrv_\idx^i - \E[\mgrv_\idx^i \mid \mc{F}_{\idx - 1}^i]$.
  To apply Lemma~\ref{lemma:freedman}, we must bound the variance of
  the difference sequence. To that end, note that the conditional
  variance is bounded as
  \begin{align*}
    \lefteqn{
      \E\left[(\mgrv_\idx^i - \E[\mgrv_\idx^i \mid \mc{F}_{\idx - 1}^i])^2
        \mid \mc{F}_{\idx - 1}^i\right]} \\
    & \le \E\left[(\mgrv_\idx^i)^2 \mid \mc{F}_{\idx - 1}^i\right] \\
    & = \E\left[\left(f(\hyp((\idx - 1)\tau + i)) - f(\hyp^*) -
      F(\hyp((\idx - 1)\tau + i); \statsample_{\tau \idx + i - 1}) + 
      F(\hyp^*;\statsample_{\idx\tau + i - 1})\right)^2
      \mid \mc{F}_{\idx - 1}^i\right] \\
    & \le 4 \lipobj^2\norm{\hyp((\idx - 1)\tau + i) - \hyp^*}^2,
  \end{align*}
  where in the last line we used the Lipschitz
  assumption~\ref{assumption:F-lipschitz} and the fact that
  $\hyp((\idx - 1) \tau + i) \in \mc{F}_{\idx - 1}^i$. Of course,
  since $\hyp^*$ minimizes $f$, the $\strongparam$-strong convexity of
  $f$ implies (see e.g.~\cite{HiriartUrrutyLe96ab}) that for any $\hyp
  \in \xdomain$, $f(\hyp) - f(\hyp^*) \ge \frac{\strongparam}{2}
  \norm{\hyp - \hyp^*}^2$. Consequently, we see that
  \begin{equation}
    \E\left[(\mgrv_\idx^i - \E[\mgrv_\idx^i \mid \mc{F}_{\idx - 1}^i]
      )^2 \mid \mc{F}_{\idx - 1}^i\right]
    \le \frac{8 \lipobj^2}{\strongparam}\left[
      f(\hyp((\idx - 1) \tau + i)) - f(\hyp^*)\right].
    \label{eqn:single-mg-variance}
  \end{equation}

  What remains is to use the single term conditional variance
  bound~\eqref{eqn:single-mg-variance} to achieve deviation control
  over the entire sequence $\mgrv_\idx^i$. To that end, recall the
  index sets $\indexset(i)$ defined in the proof of
  Theorem~\ref{theorem:highprob-error-convex}, and define the summed
  variance terms $V_i \defeq \sum_{\idx \in \indexset(i)}
  \E[(\mgrv_\idx^i - \E[\mgrv_\idx^i \mid \mc{F}_{\idx - 1}^i])^2 \mid
    \mc{F}_{\idx - 1}^i]$. Then the
  bound~\eqref{eqn:single-mg-variance} gives
  \begin{equation*}
    V_i \le \frac{8 \lipobj^2}{\strongparam}
    \sum_{\idx \in \indexset(i)} \left[f(\hyp(\tau(\idx - 1) + i))
      - f(x^*)\right].
  \end{equation*}
  Using the preceding variance bound, we can apply Freedman's
  concentration result (Lemma~\ref{lemma:freedman}) to see that with
  probability at least $1 - (4 \delta \log \ntrain) / \tau$,
  \begin{align}
    \sum_{\idx \in \indexset(i)}
    \left(\mgrv_\idx^i - \E[\mgrv_\idx^i \mid \mc{F}_{\idx - 1}^i] \right)
    & \le \max\left\{2 \sqrt{V_i}, 6 \lipobj \radius
    \sqrt{\log(\tau / \delta)}\right\} \sqrt{\log(\tau / \delta)}
    \label{eqn:mg-variance-bound}
  \end{align}

  We can use the inequality~\eqref{eqn:mg-variance-bound} to show
  concentration. Define the summations
  \begin{equation*}
    S_i \defeq \sum_{\idx \in \indexset(i)} f(\hyp(\tau(\idx - 1) +
    i)) - f(\hyp^*) ~~~ \mbox{and} ~~~ \what{S}_i \defeq \sum_{\idx
      \in \indexset(i)} F(\hyp(\tau(\idx - 1) + i); \statsample_{\tau
      \idx + i - 1}) - F(\hyp^*; \statsample_{\tau \idx + i - 1}).
  \end{equation*}
  Then the definition~\eqref{eqn:def-mg-rv} of the random variables
  $\mgrv_\idx^i$ coupled with the inequality~\eqref{eqn:mg-variance-bound}
  implies that
  \begin{align*}
    S_i & \le \what{S}_i + \max\bigg\{\sqrt{\frac{32
        \lipobj^2}{\strongparam}} \sqrt{S_i}, 6 \lipobj \radius
    \sqrt{\log\frac{\tau}{\delta}}\bigg\}
    \sqrt{\log\frac{\tau}{\delta}} + \sum_{\idx \in \indexset(i)}
    \E[\mgrv_\idx^i \mid \mc{F}_{\idx - 1}^i] \\
    & \le \what{S}_i + \sqrt{\frac{32 \lipobj^2
        \log\frac{\tau}{\delta}}{ \strongparam}} \sqrt{S_i} + 6
    \lipobj \radius \log\frac{\tau}{\delta} + |\indexset(i)|
    \phi(\tau) \radius \lipobj,
  \end{align*}
  where we have applied Lemma~\ref{lemma:lookahead-mixing}. Solving
  the induced quadratic in $\sqrt{S_i}$, we see
  \begin{equation*}
    \sqrt{S_i} \le \sqrt{\frac{8 \lipobj^2
        \log\frac{\tau}{\delta}}{\strongparam}} + \sqrt{\frac{8
        \lipobj^2}{\strongparam} \log \frac{\tau}{\delta} + \what{S}_i
      + |\indexset(i)| \phi(\tau) \radius G + 6 \lipobj \radius
      \log\frac{\tau}{\delta}}.
  \end{equation*}
  Squaring both sides and using that $(a + b)^2 \le 2a^2 + 2b^2$,
  we find that
  \begin{equation}
    S_i \le \frac{32 \lipobj^2}{\strongparam} \log\frac{\tau}{\delta}
    + 2 \what{S}_i + 12 \lipobj \radius \log\frac{\tau}{\delta} + 2
    |\indexset(i)| \phi(\tau) \radius \lipobj
    \label{eqn:si-bound}
  \end{equation}
  with probability at least $1 - 4 \delta \log \ntrain / \tau$.

  We have now nearly completed the proof of the theorem. Our first step for
  the remainder is to note that
  \begin{equation*}
    \sum_{i = 1}^\tau S_i =
    \sum_{\idx = 1}^\ntrain f(\hyp(\idx)) - f(\hyp^*)
  \end{equation*}
  Applying a union bound, we use the inequality~\eqref{eqn:si-bound}
  to see that with probability at least $1 - 4\delta \log \ntrain$,
  \begin{align*}
    \sum_{\idx = 1}^\ntrain f(\hyp(\idx)) - f(\hyp^*)
    \le 2 \sum_{i = 1}^\tau \what{S}_i
    + \frac{32 \lipobj^2 \tau}{\strongparam}
    \log \frac{\tau}{\delta}
    + 12 \tau \lipobj \radius \log\frac{\tau}{\delta}
    + 2 \ntrain \phi(\tau) \radius \lipobj.
  \end{align*}
  All that remains is to use stability to relate the sum $\sum_{i=1}^\tau
  \what{S}_i$ to the regret $\regret_\ntrain$, which is similar to what we did
  in the proof of Proposition~\ref{proposition:master-theorem}.
  Indeed, by the definition of the sums $\what{S}_i$ we have
  \begin{align}
    \sum_{i = 1}^\tau \what{S}_i
    & = \sum_{\idx = 1}^{\ntrain}
    F(\hyp(\idx); \statsample_{\idx + \tau-1})
    - F(\hyp^*; \statsample_{\idx + \tau-1}) \nonumber \\ 
    & = \sum_{\idx = 1}^\ntrain F(\hyp(\idx); \statsample_\idx) 
    - F(\hyp^*; \statsample_\idx)
    + \sum_{\idx = 1}^{\ntrain - \tau}
    F(\hyp(\idx); \statsample_{\idx + \tau-1})
    - F(\hyp(\idx + \tau-1); \statsample_{\idx + \tau-1}) \nonumber
    \\  
    & \quad ~ + \sum_{\idx = 1}^{\tau-1} F(\hyp^*; \statsample_\idx) 
    - \sum_{\idx = \ntrain + 1}^{\ntrain + \tau-1} F(\hyp^*;
    \statsample_\idx) + \sum_{\idx = \ntrain - \tau + 1}^\ntrain
    F(\hyp(\idx); \statsample_{\idx + \tau - 1})
    - \sum_{\idx = 1}^{\tau-1} F(\hyp(\idx); \statsample_\idx)
    \nonumber \\
    & \le \regret_\ntrain + 2 (\tau-1) \lipobj \radius
    + (\tau-1) \lipobj \sum_{\idx = 1}^\ntrain \stability(\idx),
    \label{eqn:sc-apply-stability}
  \end{align}
  where the inequality follows from the definition~\eqref{eqn:lowregret}
  of the regret, the boundedness assumption~\ref{assumption:F-lipschitz},
  and the stability assumption~\ref{assumption:stability}.
  Applying the final bound, we see that
  \begin{equation*}
    \sum_{\idx = 1}^\ntrain f(\hyp(\idx)) - f(\hyp^*)
    \le 2 \regret_\ntrain + 2 (\tau-1) \lipobj \sum_{\idx = 1}^\ntrain
    \stability(\idx)
    + \frac{32 \lipobj^2 \tau}{\strongparam}
    \log \frac{\tau}{\delta}
    + 12 \tau \lipobj \radius \log\frac{\tau}{\delta}
    + 2 \ntrain \phi(\tau) \radius \lipobj
    + 4 (\tau-1) \radius \lipobj
  \end{equation*}
  with probability at least $1 - 4 \delta \log \ntrain$.  Dividing by
  $\ntrain$ and applying Jensen's inequality completes the proof.
\end{proof-of-theorem}

We now turn to the case of $\beta$-mixing. As before, the proof
largely follows the proof of the $\phi$-mixing case, with a suitable
application of Markov's inequality being the only difference. 

\begin{theorem}
  \label{theorem:highprob-error-strong-beta}
  In addition to Assumptions~\ref{assumption:F-lipschitz}
  and~\ref{assumption:stability}, assume further that the expected
  function $f$ is $\strongparam$-strongly convex with respect to the
  norm $\norm{\cdot}$ over $\xdomain$.  Then for any $\delta < 1/e$,
  $\ntrain \geq 3$, with probability greater than $1-5\delta \log n$,
  for any $\tau \in \N$ the predictor $\what{\hyp}_{\ntrain}$
  satisfies
  \begin{equation*}
    f(\what{\hyp}_{\ntrain}) - f(\hyp^*) \le \frac{2}{\ntrain}
    \regret_\ntrain + \frac{2 (\tau-1) \lipobj}{\ntrain}
    \left(\sum_{\idx = 1}^\ntrain \stability(\idx) + 2\radius\right) +
    \frac{32 \lipobj^2 \tau}{\strongparam \ntrain}
    \log\frac{\tau}{\delta} + \frac{12 \tau \radius
      \lipobj}{\ntrain}\log \frac{2\tau}{\delta} + \frac{2 \radius
      \lipobj \beta(\tau)}{\delta}.
  \end{equation*}
\end{theorem}
\begin{proof}
  We closely follow the proof of Theorem~\ref{theorem:highprob-error-strong}.
  Through the bound~\eqref{eqn:mg-variance-bound}, no step in the proof of
  Theorem~\ref{theorem:highprob-error-strong} uses $\phi$-mixing.  The use of
  $\phi$-mixing occurs in bounding terms of the form $\E[\mgrv_\idx^i \mid
    \mc{F}_{\idx - 1}^i]$. Rather than bounding them immediately (as was done
  following Eq.~\eqref{eqn:mg-variance-bound} in the proof of
  Theorem~\ref{theorem:highprob-error-strong}), we carry them further through
  the steps of the proof. Using the notation of
  Theorem~\ref{theorem:highprob-error-strong}'s proof, in place of the
  inequality~\eqref{eqn:si-bound}, we have
  \begin{equation*}
    S_i \le \frac{32 \lipobj^2}{\strongparam} \log\frac{\tau}{\delta}
    + 2 \what{S}_i + 12 \lipobj \radius \log\frac{\tau}{\delta}
    + \sum_{\idx \in \indexset(i)}
    \E\left[\mgrv_{\idx}^i\mid\mc{F}_{\idx-1}^i \right]
  \end{equation*}
  with probability at least $1 - 4 \delta \log\ntrain / \tau$.
  Paralleling the proof of Theorem~\ref{theorem:highprob-error-strong},
  we find that with probability at least $1 - 4\delta \log\ntrain$,
  \begin{align}
    \lefteqn{\sum_{\idx = 1}^\ntrain f(\hyp(\idx)) - f(\hyp^*)}
    \label{eqn:sum-si-bound-beta} \\
    & \le 2 \regret_\ntrain + 2 (\tau-1) \lipobj \sum_{\idx =
      1}^\ntrain \stability(\idx) + \frac{32 \lipobj^2
      \tau}{\strongparam} \log \frac{\tau}{\delta} + 12 \tau \lipobj
    \radius \log \frac{\tau}{\delta} + 4 (\tau-1) \radius \lipobj +
    \sum_{i = 1}^\tau \sum_{\idx \in \indexset(i)}
    \E\left[\mgrv_{\idx}^i\mid\mc{F}_{\idx-1}^i \right].  \nonumber
  \end{align}
  As in the proof of Theorem~\ref{theorem:highprob-error-convex-beta}, we
  apply Markov's inequality to the final term, which gives with probability at
  least $1-\delta$
  \begin{equation*}
    \sum_{i=1}^\tau\sum_{\idx \in \indexset(i)}
    \E\left[\mgrv_{\idx}^i\mid\mc{F}_{\idx-1}^i \right] \le
    \frac{2\ntrain\lipobj\radius\beta(\tau)}{\delta}.
  \end{equation*}
  Substituting this bound into the inequality~\eqref{eqn:sum-si-bound-beta}
  and applying a union bound (noting that $\delta < \delta \log n$)
  completes the proof.
\end{proof}

As was the case for Theorem~\ref{theorem:highprob-error-convex-beta},
when the process $\statsample_1, \statsample_2, \ldots$ is
geometrically $\beta$-mixing, we can obtain a corollary of the above
result showing no essential loss of rates with respect to
geometrically $\phi$-mixing processes.  We omit details as the
technique is basically identical to that for
Corollary~\ref{cor:convex-geometric-beta}.

\section{Linear Prediction}
\label{sec:linear-predictors}

For this section, we place ourselves in the common statistical
prediction setting where the statistical samples come in pairs of the
form $(\statsample, \statlabel) \in \statsamplespace
\times \statlabelspace$, where $\statlabel$ is the label or target
value of the sample $\statsample$, and the samples are finite
dimensional: $\statsamplespace \subset \R^d$. Now we measure the
goodness of the hypothesis $\hyp$ on the example $(\statsample,
\statlabel)$ by
\begin{equation}
  \label{eqn:linear-predictor}
  F(\hyp; (\statsample, \statlabel)) = \loss(\statlabel, \<\statsample,
  \hyp\>),
  ~~~
  \loss : \statlabelspace \times \R \rightarrow \R,
\end{equation}
where the loss function $\loss$ measures the accuracy of the
prediction $\<\statsample, \hyp\>$.  An extraordinary number of
statistical learning problems fall into the above framework: linear
regression, where the loss is of the form $\loss(\statlabel,
\<\statsample, \hyp\>) = \half (\statlabel - \<\statsample,
\hyp\>)^2$; logistic regression, where $\loss(\statlabel,
\<\statsample, \hyp\>) = \log(1 + \exp(-\statlabel \<\statsample,
\hyp\>))$; boosting and SVMs all have the
form~\eqref{eqn:linear-predictor}.

The loss function~\eqref{eqn:linear-predictor} makes it clear that
individual samples cannot be strongly convex, since the linear
operator $\<\statsample, \cdot\>$ has a nontrivial null
space. However, in many problems, the expected loss function $f(\hyp)
\defeq \E_\stationary[F(\hyp; (\statsample, \statlabel))]$ is strongly
convex even though individual loss functions $F(\hyp; (\statsample,
\statlabel))$ are not. To quantify this, we now assume that
$\ltwo{\statsample} \le \statsamplebound$ for
$\underlyingmeasure$-a.e.\ $\statsample \in \statsamplespace$, and
make the following assumption on the loss:
\begin{assumption}[Linear strong convexity]
  \label{assumption:linear-predictor}
  For fixed $\statlabel$, the loss function $\loss(\statlabel, \cdot)$
  is a $\strongparam$-strongly convex and $\liploss$-Lipschitz scalar function
  over $[-\radius \statsamplebound, \radius \statsamplebound]$:
  \begin{equation*}
    \loss(\statlabel, b) \ge \loss(\statlabel, a) +
    \loss'(\statlabel, a)(b - a)
    + \frac{\strongparam}{2} (b - a)^2
    ~~~ \mbox{and} ~~~
    |\loss(\statlabel, b) - \loss(\statlabel, a)| \le \liploss |a - b|
  \end{equation*}
  for any $a, b \in \R$ with $\max\{|a|, |b|\} \le \radius \statsamplebound$.
\end{assumption}

Our choice of $\radius \statsamplebound$ above is intentional, since
$\<\statsample, \hyp\> \le \radius \statsamplebound$ by H\"older's
inequality and our compactness assumption~\eqref{eqn:compactness}.  A
few examples of such loss functions include logistic regression and
least-squares regression, the latter of which satisfies
Assumption~\ref{assumption:linear-predictor} with $\strongparam =
1$. To see that the expected loss function satisfying
Assumption~\ref{assumption:linear-predictor} is strongly convex, note
that\footnote{For notational convenience we use $\nabla F$ to denote
  either the gradient or a measurable selection from the subgradient
  set $\partial F$; this is no loss of generality.}
\begin{align}
  f(\hypother)
  & = \E_\stationary[\loss(\statlabel, \<\statsample, \hypother\>)]
  \nonumber \\
  & \ge \E_\stationary\left[\loss(\statlabel, \<\statsample, \hyp\>)
    + \loss'(\statlabel, \<\statsample, \hyp\>)(
    \<\statsample, \hypother\> - \<\statsample, \hyp\>)
    + \frac{\strongparam}{2}(\<\statsample, \hypother\>
    - \<\statsample, \hyp\>)^2
    \right] \nonumber \\
  & = \E_\stationary[F(\hyp; (\statsample, \statlabel))
    + \<\nabla F(\hyp; (\statsample, \statlabel)), \hypother - \hyp\>]
  + \frac{\strongparam}{2} \E_\stationary[\<\statsample, \hypother\>^2
    + \<\statsample, \hyp\>^2 - 2 \<\statsample, \hyp\>\<\statsample,
    \hypother\>] \nonumber \\
  & = f(\hyp) + \<\nabla f(\hyp), \hypother - \hyp\>
  + \frac{\strongparam}{2} \<\cov(\statsample) (\hyp - \hypother),
  \hyp - \hypother\>,
  \label{eqn:linear-strongly-convex}
\end{align}
where $\cov(\statsample)$ is the covariance matrix of $\statsample$
under the stationary distribution $\stationary$.  So as long as
$\lambda_{\min}(\cov(\statsample)) > 0$, we see that the expected
function $f$ is $\strongparam \cdot
\lambda_{\min}(\cov(\statsample))$-strongly convex.

If we had access to a stable online learning algorithm with small
(i.e.\ logarithmic) regret for losses of the
form~\eqref{eqn:linear-predictor} satisfying
Assumption~\ref{assumption:linear-predictor}, we could simply apply
Theorem~\ref{theorem:highprob-error-strong} and guarantee good
generalization properties of the predictor $\what{\hyp}_\ntrain$ the
algorithm outputs.
The theorem assumes only strong convexity of the expected function $f$,
which---as per our above discussion---is the case for linear prediction, so
the sharp generalization guarantee would follow from the
inequality~\eqref{eqn:linear-strongly-convex}.
However, we found it difficult to show that existing algorithms
satisfy our desiderata of logarithmic regret and stability, both of
which are crucial requirements for our results. Below, we present a
slight modification of Hazan et al.'s follow the approximate leader
(FTAL) algorithm~\cite{HazanAgKa07} to achieve the desired results.
Our approach is to essentially combine FTAL with the
Vovk-Azoury-Warmuth forecaster~\cite[Chapter 11.8]{CesaBianchiLu06},
where the algorithm uses the sample $\statsample$ to make its
prediction. Specifically, our algorithm is as follows. At iteration
$\idx$ of the algorithm, the algorithm receives $\statsample_\idx$,
plays the point $\hyp(\idx)$, suffers loss $F(\hyp(\idx);
(\statsample_\idx, \statlabel_\idx))$, then adds $\nabla F(\hyp(\idx);
(\statsample_\idx, \statlabel_\idx))$ to its collection of observed
(sub)gradients. The algorithm's calculation of $\hyp(\idx)$ at
iteration $\idx$ is
\begin{equation}
  \label{eqn:rda-vaw-update}
  \hyp(\idx) = \argmin_{\hyp \in \xdomain}\left\{\sum_{i=1}^{\idx-1}
  \<\nabla F(\hyp(i); (\statsample_i, \statlabel_i)), \hyp\>
  + \frac{\strongparam}{2} \sum_{i=1}^{\idx-1}
  \<\hyp(i) - \hyp, \statsample_i\>^2
  + \frac{\strongparam}{2} \hyp^\top (\statsample_\idx\statsample_\idx^\top 
  + \epsilon I) \hyp\right\}.
\end{equation}

The algorithm above is quite similar to Hazan et al.'s FTAL
algorithm~\cite{HazanAgKa07}, and the following proposition shows that the
algorithm~\eqref{eqn:rda-vaw-update} does in fact have logarithmic
regret (we give a proof of the proposition, which is somewhat
technical, in Appendix~\ref{sec:technical-proofs}).
\begin{proposition}
  \label{proposition:linear-predictor-regret}
  Let the sequence $\hyp(\idx)$ be defined by the
  update~\eqref{eqn:rda-vaw-update} under
  Assumption~\ref{assumption:linear-predictor}. Then for any $\epsilon > 0$
  and any sequence of samples $(\statsample_\idx, \statlabel_\idx)$,
  \begin{equation*}
    \sum_{\idx=1}^\ntrain F(\hyp(\idx); (\statsample_\idx, \statlabel_\idx))
    - F(\hyp^*; (\statsample_\idx, \statlabel_\idx)) \le
    \frac{9 \liploss^2 d}{2 \strongparam}
    \log\left(\frac{\statsamplebound^2 n}{\epsilon} + 1\right)
    + \frac{\strongparam \epsilon}{2} \ltwo{\hyp^*}^2.
  \end{equation*}
\end{proposition}
What remains is to show that a suitable form of stability holds for
the algorithm~\eqref{eqn:rda-vaw-update} that we have defined.  The
additional stability provided by using $\statsample_\idx$ in the
update of $\hyp(\idx)$ appears to be important. In the original
version~\cite{HazanAgKa07} of the FTAL algorithm, the predictor
$\hyp(\idx)$ can change quite drastically if a sample
$\statsample_\idx$ sufficiently different from the past---in the sense
that $\<\statsample_{\idx'}, \statsample_\idx\> \approx 0$ for $\idx'
< \idx$---is encountered.  In the presence of dependence between
samples, such large updates can be detrimental to performance, since
they keep the algorithm from exploiting the mixing of the stochastic
process.  Returning to our argument on stability, we recall the proof
of Theorem~\ref{theorem:highprob-error-strong}, specifically the
argument leading to the bound~\eqref{eqn:sc-apply-stability}.  We see
that the stability bound does not require the full power of
Assumption~\ref{assumption:stability}, but in fact it is sufficient
that
\begin{equation*}
  F(\hyp(\idx); (\statsample_{\idx + \tau}, \statlabel_{\idx + \tau}))
  - F(\hyp(\idx + \tau); (\statsample_{\idx + \tau}, \statlabel_{\idx + \tau}))
  \le \tau \stability(\idx),
\end{equation*}
that is, the differences in loss values are stable. To quantify the
stability of the algorithm~\eqref{eqn:rda-vaw-update}, we require two
definitions that will be useful here and in our subsequent proofs.
Define the outer product matrices
\begin{equation}
  \label{eqn:def-outprod}
  \outprodmat_\idx \defeq \sum_{i = 1}^t \statsample_i \statsample_i^\top
  ~~~ \mbox{and} ~~~
  \outprodmat_{\idx, \epsilon} \defeq \outprodmat_\idx + \epsilon I.
\end{equation}
Given a positive definite matrix $A$, the associated
Mahalanobis norm and its dual are defined as
\begin{equation*}
  \norm{\hyp}_A^2 \defeq \<A \hyp, \hyp\>
  ~~~ \mbox{and} ~~~
  \norm{\hyp}_{A^{-1}}^2 \defeq \<A^{-1} \hyp, \hyp\>.
\end{equation*}
Then the following proposition (whose proof we provide in
Appendix~\ref{sec:technical-proofs}) shows that stability holds for the
linear-prediction algorithm~\eqref{eqn:rda-vaw-update}.
\begin{proposition}
  \label{proposition:linear-stability}
  Let $\hyp(\idx)$ be generated according to the
  update~\eqref{eqn:rda-vaw-update} and let
  Assumption~\ref{assumption:linear-predictor} hold. Then for any
  $\tau \in \N$,
  \begin{align*}
    \lefteqn{F(\hyp(\idx); (\statsample_{\idx + \tau},
      \statlabel_{\idx + \tau})) - F(\hyp(\idx + \tau);
      (\statsample_{\idx + \tau}, \statlabel_{\idx + \tau}))} \\
    & \le \frac{\liploss^2}{2 \strongparam}
    \bigg(6 \tau \norm{\statsample_{\idx + \tau}}^2_{\outprodmat_{\idx + \tau,
        \epsilon}^{-1}}
    + 5 \sum_{s = 1}^{\tau - 1} \norm{\statsample_{\idx + s}}_{
      \outprodmat_{\idx + s, \epsilon}^{-1}}^2
    + 3 \norm{\statsample_\idx}_{\outprodmat_{\idx, \epsilon}^{-1}}^2
    \bigg)
  \end{align*}
\end{proposition}

We use one more observation to derive a generalization bound for the
approximate follow-the-leader update~\eqref{eqn:rda-vaw-update}.  For
any loss $\loss$ satisfying
Assumption~\ref{assumption:linear-predictor}, standard convex analysis
gives that $|\loss'(\statlabel, a)| \le \liploss$ so by
straightforward algebra (taking $a = -\radius \statsamplebound$ and $b
= \radius\statsamplebound$),
\begin{equation}
  \label{eqn:local-convexity-bound}
  2 \liploss |a - b| \ge\frac{\strongparam}{2}
  (b - a)^2,
  ~~~ \mbox{implying} ~~~
  \strongparam \le \frac{2\liploss}{\radius \statsamplebound}.
\end{equation}
Now, using Proposition~\ref{proposition:linear-stability} and the
regret bound from
Proposition~\ref{proposition:linear-predictor-regret}, we now give a
fast high-probability convergence guarantee for online algorithms
applied to linear prediction problems, such as linear or logistic
regression, satisfying
Assumption~\ref{assumption:linear-predictor}. Specifically,
\begin{theorem}
  \label{theorem:linear-generalization}
  Let $\hyp(\idx)$ be generated according to the
  update~\eqref{eqn:rda-vaw-update} with $\epsilon = 1$. Then with
  probability at least $1 - 4 \delta \log n$, for any $\tau \in \N$,
  \begin{align*}
    f(\what{\hyp}_\ntrain) - f(x^*)
    & \le \frac{\liploss^2 d}{\strongparam \ntrain}(9 + 14 \tau)
    \log\left(\statsamplebound^2 \ntrain + 1\right)
    + \frac{\strongparam}{\ntrain} \ltwo{\hyp^*}^2
    + \frac{32 \liploss^2 \statsamplebound^2 \tau}{
      \strongparam \ntrain \cdot \lambda_{\min}(\cov(\statsample))}
    \log \frac{\tau}{\delta} \\
    & \quad ~ + \frac{8 \tau \liploss^2}{\strongparam \ntrain}
    \left(3 \log \frac{\tau}{\delta} + 1\right)
    + \frac{4 \liploss^2}{\strongparam} \phi(\tau).
  \end{align*}
\end{theorem}
\begin{proof}
  Given the regret bound in
  Proposition~\ref{proposition:linear-predictor-regret}, all that
  remains is to control the stability of the algorithm. To that end,
  note that
  \begin{equation}
    \label{eqn:linear-predictor-stability}
    \sum_{\idx = 1}^{\ntrain - \tau}
    F(\hyp(\idx); (\statsample_{\idx + \tau}, \statlabel_{\idx + \tau}))
    - F(\hyp(\idx + \tau); (\statsample_{\idx + \tau},
    \statlabel_{\idx + \tau}))
    \le \frac{7 \liploss^2 \tau}{\strongparam} \sum_{\idx = 1}^\ntrain
    \norm{\statsample_\idx}_{\outprodmat_{\idx, \epsilon}^{-1}}^2
    \le \frac{7 \liploss^2 \tau d}{\strongparam} \log\left(
    \frac{\statsamplebound^2 \ntrain}{\epsilon} + 1\right),
  \end{equation}
  the last inequality following from an application of Hazan et al.'s
  Lemma~11~\cite{HazanAgKa07}.  Further, using
  Assumption~\ref{assumption:linear-predictor}, we know that the
  Lipschitz constant of $F$ is $\lipobj \le \liploss
  \statsamplebound$. We mimic the proof of
  Theorem~\ref{theorem:highprob-error-strong} for the remainder of the
  argument. This requires a minor redefinition of our martingale
  sequence, since $\hyp(\idx)$ depends on $\statsample_\idx$ in the
  update~\eqref{eqn:rda-vaw-update}, whereas our previous proofs
  required $\hyp(\idx)$ to be measurable with respect to
  $\mc{F}_{\idx-1}$. As a result, we now define
  \begin{equation*}
    \mgrv_\idx^i \defeq f(\hyp((\idx - 1) \tau + i)) - f(\hyp^*) +
    F(\hyp^*; \statsample_{\idx \tau + i}) - F(\hyp((\idx - 1) \tau
    + i); \statsample_{\idx \tau + i}),
  \end{equation*}
  and the associated $\sigma$-fields $\mc{F}_\idx^i \defeq
  \mc{F}_{\idx \tau + i} = \sigma(\statsample_1, \ldots,
  \statsample_{\idx\tau + i})$. The sequence $\mgrv_\idx^i -
  \E[\mgrv_\idx^i \mid \mc{F}_{\idx - 1}^i]$ defines a martingale
  difference sequence adapted to the filtration $\mc{F}_\idx^i$, $\idx
  = 1, 2, \ldots$. The remainder of the proof parallels that of
  Theorem~\ref{theorem:highprob-error-strong}, with the modification
  that terms involving $(\tau-1)\lipobj$ are replaced by terms
  involving $\tau\lipobj$. Specifically, we use the
  inequality~\eqref{eqn:sc-apply-stability}, the regret bound from
  Proposition~\ref{proposition:linear-predictor-regret}, and the
  stability guarantee~\eqref{eqn:linear-predictor-stability} to see
  \begin{align*}
    f(\what{\hyp}_\ntrain) - f(\hyp^*)
    & \le \frac{\liploss^2 d}{\strongparam \ntrain}(9 + 14 \tau)
    \log\left(\frac{\statsamplebound^2 \ntrain}{\epsilon} + 1\right)
    + \frac{\strongparam \epsilon}{\ntrain} \ltwo{\hyp^*}^2
    + \frac{32 \liploss^2 \statsamplebound^2 \tau}{
      \strongparam \ntrain \cdot \lambda_{\min}(\cov(\statsample))}
    \log \frac{\tau}{\delta} \\
    & \quad ~ + \frac{3 \tau \liploss \radius \statsamplebound}{\ntrain}
    \left(4 \log \frac{\tau}{\delta} + 1\right)
    + 2 \liploss \radius\statsamplebound \phi(\tau).
  \end{align*}
  Noting that $\radius \statsamplebound \le 2 \liploss / \strongparam$ by
  the bound~\eqref{eqn:local-convexity-bound} completes the proof.
\end{proof}

To simplify the conclusions of Theorem~\ref{theorem:linear-generalization},
we can ignore constants and
the size of the sample space $\statsamplespace$. Doing
this, we see that with probability at least $1 - \delta$,
\begin{equation*}
  f(\what{\hyp}_\ntrain) - f(\hyp^*)
  \le \order(1) \cdot \inf_{\tau \in \N}
  \left[\frac{\liploss^2 d \tau}{\strongparam \ntrain} \log \ntrain
    + \frac{\liploss^2 \tau}{\strongparam \ntrain \cdot
      \lambda_{\min}(\cov(\statsample))} \log \frac{\tau \log \ntrain}{\delta}
    + \frac{\liploss^2}{\strongparam} \phi(\tau)\right].
\end{equation*}
In particular, we can specialize this result in the face of different
mixing assumptions on the process. We give the bound only for
geometrically mixing processes, that is, when $\phi(k) \le \phi_0
\exp(-\phi_1 k^\mixexp)$.  Then we have---as in
Corollary~\ref{cor:convex-geometric}---the following:
\begin{corollary}
  Let $\hyp(\idx)$ be generated according to the
  follow-the-approximate leader update~\eqref{eqn:rda-vaw-update} and
  assume that the process $\statprob$ is geometrically
  $\phi$-mixing. Then with probability at least $1 - \delta$,
  \begin{equation*}
    f(\what{\hyp}_\ntrain) - f(\hyp^*) \le \order(1) \cdot
    \left[\frac{\liploss^2 d (\log \ntrain)^{1 +
          \frac{1}{\mixexp}}}{\phi_1^{1/\mixexp} \strongparam \ntrain}
      + \frac{\liploss^2 (\log \ntrain)^{\frac{1}{\mixexp}}}{
        \phi_1^{1/\mixexp} \strongparam \ntrain \cdot
        \lambda_{\min}(\cov(\statsample))} \log\left(\frac{\log
        \ntrain}{\delta}\right)\right].
  \end{equation*}
\end{corollary}

We conclude this section by noting without proof that, since all the
results here build on the theorems of
Section~\ref{sec:strongly-convex}, it is possible to analogously
derive corresponding high-probability convergence guarantees when the
stochastic process $\statprob$ is $\beta$-mixing rather than
$\phi$-mixing. In this case, we build on
Theorem~\ref{theorem:highprob-error-strong-beta} rather than
Theorem~\ref{theorem:highprob-error-strong}, but the techniques are
largely identical.

\section{Conclusions}

In this paper, we have shown how to obtain high-probability
data-dependent bounds on the generalization error, or excess risk, of
hypotheses output by online learning algorithms, even when samples are
dependent.  In doing so, we have extended several known results on the
generalization properties of online algorithms with independent
data. By using martingale tools, we have given (we hope) direct simple
proofs of convergence guarantees for learning algorithms with
dependent data without requiring the machinery of empirical process
theory. In addition, the results in this paper may be of independent
interest for stochastic optimization, since they show both the
expected and high-probability convergence of any low-regret stable
online algorithm for stochastic approximation problems, even with
dependent samples.

We believe there are a few natural open questions this work raises.
First, can online algorithms guarantee good generalization performance
when the underlying stochastic process is only $\alpha$-mixing? Our
techniques do not seem to extend readily to this more general setting,
as it is less natural for measuring convergence of conditional
distributions, so we suspect that a different or more careful approach
will be necessary. Our second question regards adaptivity: can an
online algorithm be more intimately coupled with the data and
automatically adapt to the dependence of the sequence of statistical
samples $\statsample_1, \statsample_2, \ldots$? This might allow both
stronger regret bounds and better rates of convergence than we have
achieved.

\subsection*{Acknowledgments}

We would like to thank Nicol\'o Cesa-Bianchi and several anonymous
reviewers, whose careful readings of our work greatly improved it. In
performing this work, AA
was supported in part by a MSR PhD Fellowship and a Google PhD
Fellowship, and JD was supported by the Department of Defense through
a National Defense Science and Engineering Graduate Fellowship.

\appendix

\section{Technical Proofs}
\label{sec:technical-proofs}

\begin{proof-of-proposition}[\ref{proposition:linear-predictor-regret}]
  We first give an equivalent form of the algorithm~\eqref{eqn:rda-vaw-update}
  for which it is a bit simpler to proof results (though the form is less
  intuitive). Define the (sub)gradient-like vectors $g(\idx)$ for all $t$ as
  \begin{equation}
    \label{eqn:def-g-vec}
    g(\idx) \defeq \nabla F(\hyp(\idx); (\statsample_\idx, \statlabel_\idx))
    - \strongparam \statsample_\idx \statsample_\idx^\top \hyp(\idx).
  \end{equation}
  Then a bit of algebra shows that the algorithm~\eqref{eqn:rda-vaw-update} is
  equivalent to
  \begin{equation}
    \label{eqn:linear-rda-update}
    \hyp(\idx) = \argmin_{\hyp \in \xdomain} \left\{
    \sum_{i=1}^{\idx-1} \<g(i), \hyp\> + \frac{\strongparam}{2}
    \<\outprodmat_{\idx, \epsilon} \hyp, \hyp\>\right\}.
  \end{equation}

  We now turn to the proof of the regret bound
  in the theorem. Our proof is similar to the proofs of related results of
  Nesterov~\cite{Nesterov09} and Xiao~\cite{Xiao10}.  We begin by noting that
  via Assumption~\ref{assumption:linear-predictor},
  \begin{align}
    \lefteqn{\sum_{\idx = 1}^\ntrain F(\hyp(\idx);
      (\statsample_\idx, \statlabel_\idx))
      - F(\hyp^*; (\statsample_\idx, \statlabel_\idx))} \nonumber \\
    & \le \sum_{\idx=1}^\ntrain
    \<\nabla F(\hyp(\idx); (\statsample_\idx, \statlabel_\idx)),
    \hyp(\idx) - \hyp^*\>
    - \frac{\strongparam}{2} \sum_{\idx=1}^\ntrain
    (\hyp(\idx) - \hyp^*)^\top \statsample_\idx \statsample_\idx^\top
    (\hyp(\idx) - \hyp^*) \nonumber \\
    & = \sum_{\idx=1}^\ntrain \<\nabla F(\hyp(\idx);
    (\statsample_\idx, \statlabel_\idx))
    - \strongparam \statsample_\idx \statsample_\idx^\top \hyp(\idx),
    \hyp(\idx) - \hyp^*\>
    + \frac{\strongparam}{2} \sum_{\idx=1}^\ntrain \<\statsample_\idx
    \statsample_\idx^\top \hyp(\idx), \hyp(\idx)\>
    - \frac{\strongparam}{2} \sum_{\idx=1}^\ntrain
    \<\statsample_\idx \statsample_\idx^\top
    \hyp^*, \hyp^*\> \nonumber \\
    & = \sum_{\idx=1}^\ntrain \<g(\idx), \hyp(\idx) - \hyp^*\>
    + \frac{\strongparam}{2} \sum_{\idx=1}^\ntrain
    \<\statsample_\idx, \hyp(\idx)\>^2
    - \frac{\strongparam}{2} \<\outprodmat_\ntrain \hyp^*, \hyp^*\>.
    \label{eqn:linear-rda-regret}
  \end{align}
  Define the proximal function $\prox_\idx(\hyp) = \frac{\strongparam}{2}
  \<\outprodmat_{\idx, \epsilon} \hyp, \hyp\>$ and let $z(\idx) =
  \sum_{i=1}^\idx g(i)$.  Then we can bound the
  regret~\eqref{eqn:linear-rda-regret} by taking a supremum and introducing
  the conjugate to $\prox$, defined by $\proxdual_\ntrain(z) = \sup_{\hyp \in
    \xdomain}\{\<z, \hyp\> - \prox_\ntrain(\hyp)\}$.  In particular, we see
  that for any $\epsilon \ge 0$
  \begin{align}
    \lefteqn{\sum_{\idx=1}^\ntrain
      F(\hyp(\idx); (\statsample_\idx, \statlabel_\idx)) -
      F(\hyp^*; (\statsample_\idx, \statlabel_\idx))} \nonumber \\
    & \le \sum_{\idx = 1}^\ntrain \<g(\idx), \hyp(\idx)\>
    + \frac{\strongparam}{2} \sum_{\idx=1}^\ntrain
    \<\statsample_\idx, \hyp(\idx)\>^2
    + \sup_{\hyp \in \xdomain} \left\{-\<z(\ntrain), \hyp\>
    - \frac{\strongparam}{2}
    \<\outprodmat_\ntrain \hyp, \hyp\> - \frac{\strongparam \epsilon}{2}
    \ltwo{\hyp}^2\right\} + \frac{\strongparam \epsilon}{2} \ltwo{\hyp^*}^2
    \nonumber \\
    & = \sum_{\idx=1}^\ntrain \<g(\idx), \hyp(\idx)\>
    + \frac{\strongparam}{2} \sum_{i=1}^\ntrain
    \<\statsample_\idx, \hyp(\idx)\>^2
    + \proxdual_\ntrain(-z(\ntrain)) + \frac{\strongparam\epsilon}{2}
    \ltwo{\hyp^*}^2.
    \label{eqn:linear-rda-introduce-dual}
  \end{align}
  The function $\proxdual_\ntrain$ has $(1/\strongparam)$-Lipschitz continuous
  gradient with respect to the Mahalanobis norm induced by
  $\outprodmat_{\ntrain, \epsilon}$
  (e.g.~\cite{HiriartUrrutyLe96ab,Nesterov09}), and further it is known that
  $\nabla \proxdual_\ntrain(z) = \argmin_{\hyp \in \xdomain} \{\<-z, \hyp\> +
  \prox_\ntrain(\hyp)\}$ so that $\nabla \proxdual_\ntrain(-z(\ntrain - 1)) =
  \hyp(\ntrain)$ by definition of the update~\eqref{eqn:rda-vaw-update}.
  Thus we see
  \begin{align*}
    \proxdual_\ntrain(-z(\ntrain))
    & \le \proxdual_\ntrain(-z(\ntrain - 1))
    + \<\nabla \proxdual_\ntrain(-z(\ntrain - 1)), z(\ntrain - 1) -z(\ntrain)\>
    + \frac{1}{2 \strongparam}\norm{z(\ntrain) - z(\ntrain - 1)}^2_{
      \outprodmat_{\ntrain,\epsilon}^{-1}} \\
    & = \proxdual_\ntrain(-z(\ntrain - 1))
    - \<\hyp(\ntrain), g(\ntrain)\> +
    \frac{1}{2 \strongparam}
    \norm{g(\ntrain)}^2_{\outprodmat_{\ntrain,\epsilon}^{-1}} \\
    & = -\<z(\ntrain - 1), \hyp(\ntrain)\>
    - \frac{\strongparam}{2} \<\outprodmat_{\ntrain, \epsilon}
    \hyp(\ntrain), \hyp(\ntrain)\>
    - \<\hyp(\ntrain), g(\ntrain)\> +
    \frac{1}{2 \strongparam}
    \norm{g(\ntrain)}^2_{\outprodmat_{\ntrain,\epsilon}^{-1}}.
  \end{align*}
  since $\hyp(\ntrain)$ minimizes $\<z(\ntrain - 1), \hyp\> +
  \prox_\ntrain(\hyp)$. Plugging the last inequality into the
  bound~\eqref{eqn:linear-rda-introduce-dual} yields
  \begin{align*}
    \lefteqn{\sum_{\idx=1}^\ntrain
      F(\hyp(\idx); (\statsample_\idx, \statlabel_\idx))
      - F(\hyp^*; (\statsample_\idx, \statlabel_\idx))} \\
    & \le \sum_{\idx=1}^{\ntrain} \<g(\idx), \hyp(\idx)\>
    + \frac{\strongparam}{2} \sum_{\idx=1}^\ntrain \<\statsample_\idx,
    \hyp(\idx)\>^2
    - \<z(\ntrain - 1), \hyp(\ntrain)\>
    - \frac{\strongparam}{2} \<\outprodmat_{\ntrain, \epsilon} \hyp(\ntrain),
    \hyp(\ntrain)\>
    - \<\hyp(\ntrain), g(\ntrain)\> \\
    & \qquad ~
    + \frac{\strongparam\epsilon}{2} \ltwo{\hyp^*}^2
    + \frac{1}{2\strongparam} \norm{g(\ntrain)}_{
      \outprodmat_{\ntrain,\epsilon}^{-1}}^2 \\
    & = \sum_{\idx=1}^{\ntrain-1} \<g(\idx), \hyp(\idx)\>
    + \frac{\strongparam}{2} \sum_{\idx=1}^{\ntrain-1}
    \<\statsample_\idx, \hyp(\idx)\>^2
    - \<z(\ntrain - 1), \hyp(\ntrain)\>
    - \frac{\strongparam}{2} \<\outprodmat_{\ntrain-1, \epsilon}
    \hyp(\ntrain), \hyp(\ntrain)\> \\
    & \qquad ~
    + \frac{\strongparam\epsilon}{2} \ltwo{\hyp^*}^2
    + \frac{1}{2\strongparam} \norm{g(\ntrain)}_{
      \outprodmat_{\ntrain,\epsilon}^{-1}}^2 \\
    & \le \sum_{\idx = 1}^{\ntrain - 1} \<g(\idx), \hyp(\idx)\>
    + \frac{\strongparam}{2} \sum_{\idx = 1}^{\ntrain - 1}
    \<\statsample_\idx, \hyp(\idx)\>^2
    + \proxdual_{\ntrain - 1}(-z(\ntrain - 1))
    + \frac{\strongparam \epsilon}{2} \ltwo{\hyp^*}^2
    + \frac{1}{2\strongparam} \norm{g(\ntrain)}^2_{
      \outprodmat_{\ntrain, \epsilon}^{-1}}
  \end{align*}
  since $\outprodmat_\ntrain = \outprodmat_{\ntrain-1} +
  \statsample_\ntrain \statsample_\ntrain^\top$.  Repeating the argument
  inductively down from $\ntrain - 1$, we find
  \begin{equation}
    \label{eqn:near-rda-bound}
    \sum_{\idx=1}^\ntrain F(\hyp(\idx); (\statlabel_\idx, \statsample_\idx))
    - F(\hyp^*; (\statlabel_\idx, \statsample_\idx))
    \le \frac{1}{2\strongparam} \sum_{\idx=1}^\ntrain
    \norm{g(\idx)}_{\outprodmat_{\idx,\epsilon}^{-1}}^2
    + \frac{\strongparam \epsilon}{2} \ltwo{\hyp^*}^2.
  \end{equation}

  The bound~\eqref{eqn:near-rda-bound} nearly completes the proof of the
  theorem, but we must control the gradient norm
  $\norm{g(\idx)}_{\outprodmat_{\idx,\epsilon}^{-1}}^2$ terms. To that end,
  let $\alpha_\idx = \loss'(\statlabel_\idx, \<\statsample_\idx, \hyp(\idx)\>)
  \in \R$ and note that
  \begin{equation*}
    \norm{g(\idx)}_{\outprodmat_{\idx,\epsilon}^{-1}}^2
    = \<\outprodmat_{\idx,\epsilon}^{-1}(\alpha_\idx \statsample_\idx
    - \strongparam \statsample_\idx \statsample_\idx^\top \hyp(\idx)),
    \alpha_\idx \statsample_\idx
    - \strongparam \statsample_\idx \statsample_\idx^\top \hyp(\idx)\>
    \le (\liploss + \strongparam \radius \statsamplebound)^2
    \norm{\statsample_\idx}_{\outprodmat_{\idx, \epsilon}^{-1}}^2
  \end{equation*}
  since by Assumption~\ref{assumption:linear-predictor}, $|\alpha_\idx| \le
  \liploss$. Now we apply a result of Hazan et al.~\cite[Lemma
    11]{HazanAgKa07}, giving
  \begin{equation*}
    \sum_{\idx=1}^\ntrain \norm{g(\idx)}_{\outprodmat_{\idx,\epsilon}^{-1}}^2
    \le (\liploss + \strongparam \radius \statsamplebound)^2
    d \log\left(\frac{\statsamplebound^2 n}{\epsilon} + 1\right).
  \end{equation*}
  Using that $\strongparam \le 2 \liploss / (\radius \statsamplebound)$, we
  combine this with the bound~\eqref{eqn:near-rda-bound} to get
  the result of the theorem.
\end{proof-of-proposition}

\begin{proof-of-proposition}[\ref{proposition:linear-stability}]
  We begin by noting that any $g \in \partial F(\hyp(\idx);
  (\statsample_{\idx + \tau}, \statlabel_{\idx + \tau}))$ can
  be written as $\alpha \statsample_{\idx + \tau}$ for some
  $\alpha \in [-\liploss, \liploss]$. Thus, 
  using the first-order convexity inequality, we see there
  is such an $\alpha$ for which
  \begin{equation*}
    F(\hyp(\idx); \statsample_{\idx + \tau}) -
    F(\hyp(\idx + \tau); \statsample_{\idx + \tau})
    \le \alpha \<\statsample_{\idx + \tau}, \hyp(\idx) - \hyp(\idx + \tau)\>.
  \end{equation*}
  Now we apply H\"older's inequality and Lemma~\ref{lemma:near-stability},
  which together yield
  \begin{align*}
    \lefteqn{\<\statsample_{\idx + \tau}, \hyp(\idx) - \hyp(\idx + \tau)\>} \\
    & \le \norm{\statsample_{\idx + \tau}}_{
      \outprodmat_{\idx + \tau, \epsilon}^{-1}}
    \norm{\hyp(\idx) - \hyp(\idx + \tau)}_{
      \outprodmat_{\idx + \tau, \epsilon}} \\
    & \le
    \frac{3 \liploss}{\strongparam}
    \sum_{s = 0}^{\tau - 1}
    \norm{\statsample_{\idx + \tau}}_{\outprodmat_{\idx + \tau, \epsilon}^{-1}}
    \norm{\statsample_{\idx + s}}_{\outprodmat_{\idx + s,
        \epsilon}^{-1}}
    + \frac{2 \liploss}{\strongparam}
    \sum_{s = 1}^\tau
    \norm{\statsample_{\idx + \tau}}_{\outprodmat_{\idx + \tau, \epsilon}^{-1}}
    \norm{\statsample_{\idx + s}}_{\outprodmat_{\idx + s, \epsilon}^{-1}} \\
    & \le \frac{3 \liploss}{2 \strongparam}
    \bigg[\sum_{s = 0}^{\tau - 1}
    \norm{\statsample_{\idx + s}}^2_{\outprodmat_{\idx + s, \epsilon}^{-1}}
    + \tau \norm{\statsample_{\idx + \tau}}^2_{
      \outprodmat_{\idx + \tau, \epsilon}^{-1}}
    \bigg]
    + \frac{\liploss}{\strongparam}
    \bigg[\sum_{s = 1}^{\tau }
    \norm{\statsample_{\idx + s}}^2_{\outprodmat_{\idx + s, \epsilon}^{-1}}
    + \tau \norm{\statsample_{\idx + \tau}}^2_{
      \outprodmat_{\idx + \tau, \epsilon}^{-1}}
    \bigg]
  \end{align*}
  where we have used the fact that $(a^2 + b^2) / 2 \ge ab$ for any
  $a, b \in \R$.
  A re-organization of terms and using the fact that $|\alpha| \le \liploss$
  completes the proof.
\end{proof-of-proposition}

\begin{lemma}
  \label{lemma:near-stability}
  Let $\hyp(\idx)$ be generated according to the
  update~\eqref{eqn:rda-vaw-update}.  Then for any $\tau \in \N$,
  \begin{equation*}
    \norm{\hyp(\idx) - \hyp(\idx + \tau)}_{\outprodmat_{\idx + \tau, \epsilon}}
    \le \frac{3 \liploss}{\strongparam}
    \sum_{s = 0}^{\tau - 1}
    \norm{\statsample_{\idx + s}}_{\outprodmat_{\idx + s, \epsilon}^{-1}}
    + \frac{2 \liploss}{\strongparam}
    \sum_{s = 1}^\tau \norm{\statsample_{\idx + s}}_{\outprodmat_{
        t + s, \epsilon}^{-1}}.
  \end{equation*}
\end{lemma}
\begin{proof}
  Recall the definition~\eqref{eqn:def-outprod} of the outer product matrices
  $\outprodmat_\idx$ and the construction~\eqref{eqn:def-g-vec} of the
  subgradient vectors $g(\idx)$ from the proof of
  Proposition~\ref{proposition:linear-predictor-regret}. With the definition
  $z(\idx) = \sum_{i=1}^t g(i)$, also as in
  Proposition~\ref{proposition:linear-predictor-regret}, it the
  update~\eqref{eqn:rda-vaw-update} is equivalent to
  \begin{equation}
    \label{eqn:local-sc-da-update}
    \hyp(\idx) = \argmin_{\hyp \in \xdomain}
    \left\{\<z(\idx - 1), \hyp\> + \frac{\strongparam}{2}
    \<\outprodmat_{\idx,\epsilon} \hyp, \hyp\>\right\}.
  \end{equation}
  Now, let us understand the stability of the solutions to the above
  updates. Fixing $\tau \in \N$, the first order conditions for the optimality
  of $\hyp(\idx + 1)$ in the update~\eqref{eqn:local-sc-da-update} for
  $\hyp(\idx)$ and $\hyp(\idx + \tau)$ imply
  \begin{align*}
    \<z(\idx + \tau - 1) +
    \strongparam \outprodmat_{\idx + \tau, \epsilon} \hyp(\idx + \tau),
    \hyp - \hyp(\idx + \tau)\> \ge 0
    ~~~ \mbox{and} ~~~
    \<z(\idx - 1) + \strongparam \outprodmat_{\idx, \epsilon} \hyp(\idx),
    \hyp' - \hyp(\idx)\> \ge 0,
  \end{align*}
  for all $\hyp, \hyp' \in \xdomain$. Taking $\hyp = \hyp(\idx)$ and $\hyp' =
  \hyp(\idx + \tau)$, then adding the two inequalities, we see
  \begin{equation}
    \label{eqn:first-order-summed-optimality}
    \<z(\idx + \tau - 1) - z(\idx - 1) +
    \strongparam \outprodmat_{\idx + \tau, \epsilon} \hyp(\idx + \tau)
    - \strongparam \outprodmat_{\idx, \epsilon} \hyp(\idx),
    \hyp(\idx) - \hyp(\idx + \tau)\> \ge 0.
  \end{equation}

  The remainder of the proof consists of manipulating the
  inequality~\eqref{eqn:first-order-summed-optimality} to achieve the desired
  result. To begin, we rearrange
  Eq.~\eqref{eqn:first-order-summed-optimality} to state
  \begin{align*}
    \lefteqn{\<z(\idx + \tau - 1) - z(\idx - 1),
      \hyp(\idx) - \hyp(\idx + \tau)\>} \\
    & \ge \strongparam \<\outprodmat_{\idx + \tau, \epsilon}
    (\hyp(\idx) - \hyp(\idx + \tau)), \hyp(\idx) - \hyp(\idx + \tau)\>
    + \strongparam\<(\outprodmat_{\idx, \epsilon}
    - \outprodmat_{\idx + \tau, \epsilon}) \hyp(\idx),
    \hyp(\idx) - \hyp(\idx + \tau)\> \\
    & = \strongparam
    \norm{\hyp(\idx) - \hyp(\idx + \tau)}_{
      \outprodmat_{\idx + \tau, \epsilon}}^2
    + \strongparam\<(\outprodmat_{\idx, \epsilon}
    - \outprodmat_{\idx + \tau, \epsilon}) \hyp(\idx),
    \hyp(\idx) - \hyp(\idx + \tau)\>.
  \end{align*}
  Using H\"older's inequality applied to the dual norms
  $\norm{\cdot}_{\outprodmat}$ and $\norm{\cdot}_{\outprodmat^{-1}}$,
  we see that
  \begin{align*}
    \lefteqn{\strongparam \norm{\hyp(\idx) -
        \hyp(\idx + \tau)}_{\outprodmat_{\idx + \tau, \epsilon}}^2} \\
    & \le \norm{z(\idx + \tau - 1) - z(\idx - 1)}_{
      \outprodmat_{\idx + \tau, \epsilon}^{-1}}
    \norm{\hyp(\idx) - \hyp(\idx + \tau)}_{
      \outprodmat_{\idx + \tau, \epsilon}} \\
    & \qquad ~
    + \strongparam \norm{(\outprodmat_{\idx + \tau, \epsilon} -
      \outprodmat_{\idx, \epsilon}) \hyp(\idx)}_{
      \outprodmat_{\idx + \tau, \epsilon}^{-1}}
    \norm{\hyp(\idx) - \hyp(\idx + \tau)}_{\outprodmat_{\idx + \tau, \epsilon}}
  \end{align*}
  and dividing by $\strongparam \norm{\hyp(\idx) - \hyp(\idx + \tau)}$ gives
  \begin{equation}
    \label{eqn:intermediate-stability}
    \norm{\hyp(\idx) - \hyp(\idx + \tau)}_{\outprodmat_{\idx + \tau, \epsilon}}
    \le \frac{1}{\strongparam}
    \norm{z(\idx + \tau - 1) - z(\idx - 1)}_{
      \outprodmat_{\idx + \tau, \epsilon}}
    + \norm{(\outprodmat_{\idx + \tau, \epsilon} -
      \outprodmat_{\idx, \epsilon})
      \hyp(\idx)}_{\outprodmat_{\idx + \tau, \epsilon}^{-1}}.
  \end{equation}
  Now we note the fact that
  $\outprodmat_{\idx + \tau, \epsilon} - \outprodmat_{\idx, \epsilon}
  = \sum_{s = 1}^\tau \statsample_{\idx + s} \statsample_{\idx + s}^\top$,
  so
  \begin{equation*}
    \norm{(\outprodmat_{\idx + \tau, \epsilon} - \outprodmat_{\idx, \epsilon})
      \hyp(\idx)}_{A_{\idx + \tau, \epsilon}^{-1}}
    \le \max_{s \in [\tau]} |\<\statsample_{\idx + s}, \hyp(\idx)\>|
    \sum_{s = 1}^\tau \norm{\statsample_{\idx + s}}_{\outprodmat_{\idx + \tau,
        \epsilon}^{-1}}
    \le \radius \statsamplebound \sum_{s = 1}^\tau
    \norm{\statsample_{\idx + s}}_{\outprodmat_{\idx + \tau, \epsilon}^{-1}}.
  \end{equation*}
  In addition, we have $z(\idx + \tau - 1) - z(\idx - 1) = \sum_{s = 0}^{\tau
    - 1} g(\idx + s)$, and as in the proof of
  Proposition~\ref{proposition:linear-predictor-regret},
  \begin{equation*}
    \norm{z(\idx + \tau - 1) - z(\idx - 1)}_{
      \outprodmat_{\idx + \tau, \epsilon}^{-1}}
    \le (\liploss + \strongparam \radius \statsamplebound)
    \sum_{s = 0}^{\tau - 1} \norm{\statsample_{\idx + s}}_{
      \outprodmat_{\idx + \tau, \epsilon}^{-1}}
    \le 3 \liploss \sum_{s = 0}^{\tau - 1} \norm{\statsample_{\idx + s}}_{
      \outprodmat_{\idx + \tau, \epsilon}^{-1}},
  \end{equation*}
  where for the last inequality we used the
  bound~\eqref{eqn:local-convexity-bound}, which implies
  $\radius \statsamplebound \le \frac{2 \liploss}{\strongparam}$.
  Thus the inequality~\eqref{eqn:intermediate-stability}
  yields
  \begin{equation*}
    \norm{\hyp(\idx) - \hyp(\idx + \tau)}_{\outprodmat_{\idx + \tau, \epsilon}}
    \le \frac{3 \liploss}{\strongparam} \sum_{s = 0}^{\tau - 1}
    \norm{\statsample_{\idx + s}}_{\outprodmat_{\idx + \tau, \epsilon}^{-1}}
    + \frac{2 \liploss}{\strongparam}
    \sum_{s = 1}^\tau \norm{\statsample_{\idx + s}}_{\outprodmat_{\idx + \tau,
        \epsilon}^{-1}}.
  \end{equation*}
  Noting that $\outprodmat_{\idx + 1, \epsilon} \succeq \outprodmat_{\idx,
    \epsilon}$ completes the proof.
\end{proof}

\setlength{\bibsep}{6pt} %

\bibliographystyle{abbrv}
\bibliography{bib}

\end{document}